\documentclass[11pt]{article}

\usepackage{fullpage,times,url,bm}

\usepackage{amsthm,amsfonts,amsmath,amssymb,epsfig,color,float,graphicx,verbatim}
\usepackage{algorithm,algorithmic}
\usepackage{bbm}
\usepackage[authoryear]{natbib}

\usepackage{hyperref}
\hypersetup{
	colorlinks   = true, 
	urlcolor     = blue, 
	linkcolor    = blue, 
	citecolor   = black 
}
\usepackage{bbm}
\newtheorem{theorem}{Theorem}[section]

\newtheorem{lemma}[theorem]{Lemma}
\newtheorem{corollary}[theorem]{Corollary}

\newtheorem{remark}[theorem]{Remark}

\newcommand{\reals}{\mathbb{R}}

\newcommand{\sign}{\mathrm{sign}}

\newcommand{\bx}{\mathbf{x}}
\newcommand{\bw}{\mathbf{w}}

\newcommand{\bb}{\mathbf{b}}

\newcommand{\bz}{\mathbf{z}}

\newcommand{\bh}{\mathbf{h}}

\newcommand{\Dcal}{\mathcal{D}}

\newcommand{\Ncal}{\mathcal{N}}

\newcommand{\norm}[1]{\|#1\|}
\newcommand{\inner}[1]{\langle#1\rangle}

\newcommand{\secref}[1]{Sec.~\ref{#1}}
\newcommand{\subsecref}[1]{Subsection~\ref{#1}}

\renewcommand{\eqref}[1]{Eq.~(\ref{#1})}
\newcommand{\lemref}[1]{Lemma~\ref{#1}}
\newcommand{\corollaryref}[1]{Corollary~\ref{#1}}
\newcommand{\thmref}[1]{Thm.~\ref{#1}}

\newcommand{\appref}[1]{Appendix~\ref{#1}}

\newcommand{\naturals}{\mathbb{N}}

\newcommand{\bin}{\textsc{bin}}
\newcommand{\len}{\textsc{len}}
\makeatletter
\newcommand{\printfnsymbol}[1]{%
  \textsuperscript{\@fnsymbol{#1}}%
}
\makeatother

\title{Width is Less Important than Depth in ReLU Neural Networks}
\date{}
\author{Gal Vardi\thanks{equal contribution} \qquad Gilad Yehudai\printfnsymbol{1} \qquad Ohad Shamir\\
Weizmann Institute of Science\\ 
\texttt{\{gal.vardi,gilad.yehudai,ohad.shamir\}@weizmann.ac.il}
}
\begin{document}
\maketitle
\begin{abstract}
    We solve an open question from \cite{lu2017expressive}, by showing that any target network with inputs in $\reals^d$ can be approximated by a width $O(d)$ network (independent of the target network's architecture), whose number of parameters is essentially larger only by a linear factor. In light of previous depth separation theorems, which imply that a similar result cannot hold when the roles of width and depth are interchanged, it follows that depth plays a more significant role than width in the expressive power of neural networks.
    We extend our results to constructing networks with bounded weights, and to constructing networks with width at most $d+2$, which is close to the minimal possible width due to previous lower bounds. Both of these constructions cause an extra polynomial factor in the number of parameters over the target network. We also show an exact representation of wide and shallow networks using deep and narrow networks which, in certain cases, does not increase the number of parameters over the target network.
\end{abstract}

\section{Introduction}
The expressive power of neural networks has been widely studied in many previous works. A particular focus was given to the role of the network's depth and width: How wide or how deep do we need to make the network, in order to express various target functions of interest? In an asymptotic sense, we know that making \emph{either} the width or the depth large enough is sufficient to approximate any target function of interest. Specifically, classical universal approximation results (e.g. \cite{cybenko1989approximation,leshno1993multilayer, hornik1989multilayer}) imply that even with depth $2$, a wide enough neural network can approximate essentially any target function on a bounded domain in $\reals^d$. More recently, it was shown that the same also applies to depth: Neural networks with width $\Omega(d)$ and sufficient depth can also approximate essentially any target function (e.g. \cite{lu2017expressive}). 

However, these results are asymptotic in nature, and do not provide a quantitative answer as to whether depth or width play a more significant role in the expressive power of neural networks.
A recent line of works have shown that for certain target functions, depth plays a more significant role than width, in the sense that slightly decreasing the depth requires a huge increase in the width to maintain approximation accuracy. For example, \cite{eldan2016power, safran2017depth, daniely2017depth}  constructed functions on $\reals^d$  that can be expressed by depth-$3$ neural networks with $\text{poly}(d)$ parameters, while depth-$2$ neural networks require a number of parameters at least exponential in $d$ to approximate them well. In \cite{telgarsky2016benefits, chatziafratis2019depth} a family of functions represented by depth $O(k)$ and width $O(1)$ neural networks is constructed such that approximating them up to arbitrarily small accuracy with depth $O(\sqrt{k})$ would require width 
exponential in $k$.

A natural question that arises is whether we can provide similar results in terms of width, namely:
\begin{quote}
\emph{Are there functions that can be expressed by wide and shallow neural networks, that cannot be approximated by any narrow neural network, unless its depth is very large?}
\end{quote}

This question was stated as an open problem in  \cite{lu2017expressive}. We note that both a positive and a negative answer to this question has interesting consequences. If the answer is positive, then width and depth, in principle, play an incomparable role in the expressive power of neural networks, as sometimes depth can be more significant, and sometimes width. On the other hand, if the answer is negative, then depth generally plays a more significant role than width for the expressive power of neural networks.


In this work we solve this open problem for ReLU neural networks, by providing a negative answer to the above question. In more details, we prove the following theorem:

\begin{theorem}[Informal]\label{thm:main informal}
Let $\Ncal_0:\reals^d\rightarrow\reals$ be a ReLU neural network with width $n$, depth $L$, and let $\Dcal$ be some input distribution with an upper bounded density function over a bounded domain in $\reals^d$. Then, for every $\epsilon,\delta > 0$ there exists a neural network $\Ncal:\reals^d\rightarrow\reals$ with width $O(d)$, and $\tilde{O}\left(n^2L^2\right)$ parameters such that with probability at least $1-\delta $ over $\bx\sim \Dcal$ we have:
\[
\left|\Ncal_0(\bx) - \Ncal(\bx)\right| \leq \epsilon
\]
where the $\tilde{O}$ notation hides logarithmic terms in the problem's parameters (see \thmref{thm:wide to narrow} for a formal claim).
\end{theorem}
Note that a network with width $n$ and depth $L$ has $O\left(n^2L\right)$ parameters, whereas the theorem above proves the existence of an approximating narrow network with $\tilde{O}\left(n^2L^2\right)$ parameters. This means that any wide network can be approximated up to an arbitrarily small accuracy, by a narrow network where the number of parameters increases (up to log factors) only by a factor of $L$. Hence, it shows that the price for making the width small is only a linear increase in the network depth, in sharp contrast to the results mentioned earlier on how making the depth small may require an exponential increase in the network depth. In \subsecref{subsec:num param} we further discuss the extra $L$ factor, which occurs due to a rough estimate of the Lipschitz parameter of the network. We argue that this factor can also be avoided by having stronger assumptions on the Lipschitz parameter of the networks, which shows that only having a logarithmic blow-up in the number of parameters is enough for such 
cases.

In \cite{lu2017expressive} it was shown that the universal approximation property on a compact domain does not hold for network with width less than $d$.
In \cite{park2020minimum} it was shown that networks with width $d+1$ already have the universal approximation property.
We extend our construction from \thmref{thm:main informal} to approximating any wide network using a network with width $d+2$, close to the minimal possible width. This construction has an additional blow-up from the bound in \thmref{thm:main informal} on the number of parameters by a factor of $d$. We also discuss how to extend \thmref{thm:main informal} when the construction is restricted to having bounded weights. We show that we can approximate a wide network using a narrow network with weights bounded by $O(1)$, while, suffering an additional blow-up from the bound in  \thmref{thm:main informal} by a factor of $\tilde{O}(n\cdot L)$.

The above constructions only apply on a bounded domain, and they approximate the target network w.h.p over some distribution. We additionally provide a different construction which \emph{exactly} represents for all $\bx\in\reals^d$ a target network of width $n$ and depth $L$, using a network of width $O(d+L)$, although its depth is $O\left(n^{L-1}\right)$. We show that for $d= O(1)$ and $L=2,3$, the number of parameters in this construction does not increase in comparison to the target network. Hence, for theses cases, this exact representation is more efficient by a log factor in terms of parameters than the construction in \thmref{thm:main informal}

\subsection*{Related Work}

\paragraph{The benefits of depth.}

Quite a few theoretical works in recent years have explored the beneficial effect of depth on increasing the expressiveness of neural networks.
A main focus is on {\em depth separation}, namely, showing that there is a function $f:\reals^d \rightarrow \reals$ that can be approximated by a $\text{poly}(d)$-sized network of a given depth, with respect to some input distribution, but cannot be approximated by $\text{poly}(d)$-sized networks of a smaller depth.
As we already mentioned, depth separation between depth $2$ and $3$ was shown in \citep{eldan2016power,safran2017depth,daniely2017depth}.
A construction shown by \cite{telgarsky2016benefits} gives separation between networks of a constant depth and networks of some non-constant depth. 
Complexity-theoretic barriers to proving separation between two constant depths beyond depth $4$, and to proving separation for certain ``well behaved" functions were established in \cite{vardi2020neural,vardi2021size}. 
In \cite{safran2017depth,liang2016deep,yarotsky2017error} another notion of depth separation is considered. They show that there are functions that can be $\epsilon$-approximated by a network of $\text{polylog}(1/\epsilon)$ width and depth, but cannot be $\epsilon$-approximated by a network of $O(1)$ depth unless its width is $\text{poly}(1/\epsilon)$. 
Depth separation was also widely studied in other works in recent years (e.g., \cite{martens2013representational,safran2019depth, chatziafratis2019depth,bresler2020sharp,venturi2021depth,malach2021connection}).

The expressivity benefits of depth in the context of the VC-dimension (namely, how the VC dimension increases with more depth, even if the total number of parameters remain the same) is implied by, e.g., \cite{bartlett2019nearly}. Finally, \cite{park2020provable,vardi2021optimal} proved that deep networks have more memorization power than shallow ones. That is, deep networks can memorize $N$ samples using roughly $\sqrt{N}$ parameters, while shallow networks require $N$ parameters.

\paragraph{Deep and narrow networks.}

The expressive power of narrow neural networks has been extensively studied in recent years (e.g., \citep{lu2017expressive,hanin2017approximating,johnson2018deep,kidger2020universal,park2020minimum}).
As we already discussed, \cite{lu2017expressive} posed the open question that we study in this work. They also showed that the \emph{minimal width for universal approximation} (denoted $w_\text{min}$) using ReLU networks w.r.t. the $L^1$ norm of functions from $\reals^d$ to $\reals$, satisfies $d+1 \leq w_\text{min} \leq d+4$. For $L^1$-approximation of functions from a compact domain they showed a lower bound of $w_\text{min} \geq d$. 
\cite{kidger2020universal} extended their results to $L^p$-approximation of functions from $\reals^d$ to $\reals^{d_{\text{out}}}$ , and obtained $w_\text{min} \leq d + d_{\text{out}} + 1$. 
\cite{park2020minimum} further improved this result and obtained $w_\text{min} = \max\{d+1,d_{\text{out}}\}$.
\cite{hanin2017approximating} considered universal approximation (using ReLU networks) of functions from a compact domain to $\reals^{d_{\text{out}}}$ w.r.t. the $L^{\infty}$ norm, and proved that $d+1 \leq w_\text{min} \leq d+d_{\text{out}}$.
Universal approximation using narrow networks with other activation functions has been studied in \cite{johnson2018deep,kidger2020universal,park2020minimum}. 
We note that all prior results on universal approximation using deep and narrow networks require networks of depth exponential in the input dimension. However, our results are of a different nature, since we focus on approximating a given network of bounded size, while universal approximation results aim at approximating arbitrary functions. 
For a more detailed discussion on related prior works see \cite{park2020minimum}.

\section{Preliminaries}
For $n\in\naturals$ and $i \leq j$ we denote by $\bin_{i:j}(n)$ the string of bits in places $i$ until $j$ inclusive, in the binary representation of $n$ and treat is as an integer (in binary basis). For example, $\bin_{1:3}(32)=4$, i.e. the three most significant bits (from the left). We denote by $\len(n)$ the minimal number of bits in its binary representation. We denote $\bin_i(n):= \bin_{i:i}(n)$, i.e. the $i$-th bit on $n$. For a function $f$ and $i\in\naturals$ we denote by $f^{(i)}$ the composition of $f$ with itself $i$ times. We denote vectors in bold face. For a vector $\bx$ we denote by $x_i$ its $i$-th  coordinate. We use the $\tilde{O}(\cdot)$ notation to hide logarithmic factors, and use $O(\cdot)$ to hide constant factors. For $n\in\naturals$ we denote $[n]:=\{1,\dots,n\}$. 

\subsubsection*{Neural Networks}
We denote by $\sigma(z):=\max\{0,z\}$ the ReLU function. In this paper we only consider neural networks with the ReLU activation. 

Let $d\in\naturals$ be the data input dimension. We define a \emph{neural network} of depth $L$ as $\mathcal{N}:\reals^d\rightarrow\reals$, where $\mathcal{N}(\bx)$ is computed recursively by
\begin{itemize}
    \item $\bh^{(1)} = \sigma\left(W^{(1)}\bx + \bb^{(1)}\right)$ for $W^{(1)}\in\reals^{n_1\times d}, \bb^{(1)}\in\reals^{n_1}$
    \item $\bh^{(i)} = \sigma\left(W^{(i)}\bh^{(i-1)} + \bb^{(i)}\right)$ for $W^{(i)}\in\reals^{n_{i}\times n_{i-1}}, \bb^{(i)}\in\reals^{n_i}$ for $i=2,\dots,L-1$
    \item $\mathcal{N}(\bx)
    = \bh^{(L)} 
    = W^{(L)}\bh^{(L-1)} + \bb^{(L)}$ for $W^{(L)}\in\reals^{1\times n_{L-1}}, \bb^{(L)}\in\reals^{1}$
\end{itemize}

The \emph{width} of the network is $n:=\max\{d,n_1,\ldots,n_{L-1}\}$.
We define the \emph{number of parameters} of the network as the total number of coordinates in its weight matrices $W^{(i)}$ and biases $\bb^{(i)}$, which is at most $O(n^2\cdot L)$. Note that in some previous works (e.g. \cite{bartlett2019nearly, vardi2021optimal}) the number of parameters of the network is defined as
the number of weights of $\mathcal{N}$ which are non-zero. Our definition is stricter, as we also count zero weights.

\subsection*{Input Dimension}
Throughout the paper, we assume that $d\leq n$, i.e. the input dimension is smaller than the width of the target network. This assumption is important, because our goal is to approximate a network of width $n$ and depth $L$ with a deep network, but with width bounded by $O(d)$. If $d>n$, then the network we are given is already in the correct form and there is nothing to prove. We note that networks with width smaller than $d$ do not have the universal approximation property (see e.g. \cite{lu2017expressive, johnson2018deep,park2020minimum,hanin2017approximating}), no matter how deep they are. This is in contrast to networks with depth $2$, which have the universal approximation property (where the width is unbounded). This means that we cannot expect to approximate \emph{all} networks of width $n$ using networks with width smaller than $d$, hence constructing a network with width that depends on $d$ is unavoidable.

\section{Narrow and Deep Networks Can Approximate Wide Networks}\label{sec:narrow to deep}

In this section we show that given a network of width $n$, depth $L$ and input dimension $d$, we can approximate it up to error $\epsilon$ using another network with width $O(d)$ and depth $\tilde{O}(n^2L^2)$.

\begin{theorem}\label{thm:wide to narrow}
Let $A,B,n,L,d\in\naturals,~ \epsilon,\delta>0$ and let $\Ncal_0:[-A,A]^d\rightarrow\reals$ be a neural network with width $n$, depth $L$ and weights bounded in $[-B,B]$. Let $\Dcal$ be some distribution over $[-A,A]^d$ with density function $p_\Dcal$ such that $p_\Dcal(\bx) \leq \beta$ for every $\bx\in[-A,A]^d$ where $\beta >0$. Then, there exists a neural network $\Ncal:[-A,A]^d\rightarrow\reals$ with width $\max\{5d,10\}$, depth $O\left(n^2L^2\log(ABn\epsilon^{-1})\right)$, such that $w.p > 1-\delta$ over $\bx\sim \Dcal$ we have that:
\[
|\Ncal(\bx) - \Ncal_0(\bx)| \leq \epsilon~.
\]
The total number of parameters in $\Ncal$ is $O\left(n^2L^2\log(ABn\epsilon^{-1})\right)$.
\end{theorem}

The full proof can be found in \appref{appen:proof of main theorem}. We note that the number of parameters in the target network $\Ncal_0$ is $O(n^2L)$. Hence, the number of parameters in $\Ncal$ is larger only by a factor of $\tilde{O}(L)$, we will discuss this dependence later on. 
Specifically, the blow-up in the number of parameters w.r.t. the width $n$ is only logarithmic.
The dependence on the approximation parameter $\epsilon$ is also logarithmic.  
Note that $\delta$ and $\beta$ do not affect the number of parameters in the network, as they only appear in the magnitude of the weights (see \thmref{thm:wide to narrow appen} in the Appendix, and the discussion in \subsecref{subsec:bounded weights}). 
We also note that although our result shows an approximation w.h.p, it can be easily modified to obtain approximation w.r.t $L^p$ norms. This can be done by adding an extra layer that clips large output values of the network, and once the outputs are bounded, we can choose $\delta$ accordingly to get an approximation in $L^p$.


\subsection{Proof Intuition}
The main idea for our proof is to encode for each layer (including the first layer) all its input coordinates into a single number. Now, when we want to apply some computation on an input coordinate (e.g. multiply it by a constant), we 
extract only the relevant bits out of that number and apply our computation on them.
Encoding a vector of dimension $k$ into a single number can be done in the following way: 
For each coordinate we extract its $c$ most significant bits (for an appropriate $c$), and we concatenate all these bits into a single number with a total of $k\cdot c$ bits. Now, to apply some computation on the $i$-th coordinate, we first extract the bits in places $i\cdot c$ until $(i+1)\cdot c$ from the number we created, and apply the computation on these bits. The main novelty of the proof comes from using this encoding technique,
such that
given a layer with $n$ input coordinates and $n$ output coordinates, we simulate it with a network of width $O(1)$, and depth which depends on $n^2$, and the number of extracted bits from each coordinate.
%
%
We now explain in more details the different building blocks of our proof.

\subsubsection*{Encoding the Input}
The main idea in this part is to construct a subnetwork which encodes all the $d$ coordinates of the input within a single number. For simplicity, we assume here that the inputs are in $[0,1]^d$. We construct a network $F_\text{enc}:\reals^d\rightarrow\reals$ such that for every $i\in[d]$:
\[
\bin_{(i-1)\cdot c + 1: i\cdot c}\left(F_\text{enc}(\bx)\right) = \lfloor x_i\cdot 2^c\rfloor ~.
\]
In words, each $c$ bits of the output of the network $F_{\text{enc}}$ is an encoding of the $c$ most significant bits of the $i$-th coordinate of the input. Note that if we want to approximate the input up to an error of $\epsilon$, we need to use only the $O\left(\log\left(\frac{1}{\epsilon}\right)\right)$ most significant bits. This construction uses an efficient bit extraction technique based on Telgarsky's triangle function \citep{telgarsky2016benefits}, which was also used in \cite{vardi2021optimal}. The depth of this network depends only on the number of extracted bits, and the width depends on the dimension of the inputs. We note that exact bit extraction is not a continuous operation. We approximate this operation using the ReLU activation, such that it succeeds with probability at least $1-\delta$. This $\delta$ parameter only affects the size of the weights, and not the number of parameters. We further discuss the size of the required weights of the network in Subsection \ref{subsec:bounded weights}.

\subsubsection*{Encoding Each Layer}
In this part, we construct deep and narrow subnetworks $F_\ell:\reals\rightarrow\reals$ with real-valued inputs and outputs for $\ell\in[L]$, where each such subnetwork simulates the $\ell$-th layer from the target network. We first explain how to simulate a single neuron, and then how to extend it to simulating a layer. 

A single ReLU neuron is a function of the form $\bz\mapsto \sigma(\inner{\bw,\bz} + b)$, 
for some
$\bw \in\reals^n, b\in\reals$. Suppose that the input of this single neuron (i.e. $\bz$) is represented in a single coordinate with $c\cdot n$ bits, where each $c$ bits represents the $c$ most significant bits of a coordinate of $\bz$.
We iteratively decode the representation of $z_i$ (the $i$-th coordinate of $\bz$), multiply it by $w_i$ and add it to a designated output number. 
The decoding of the input is done using Telgarsky's triangle function. To deal with both negative and positive $w_i$'s, we use two designated output numbers, one for the positive weights and one for the negative weights. The final layer adds up these designated outputs with their corresponding sign to get the correct result. 
Then, it adds the bias term $b$ to the output and applies the ReLU function.
The depth of this network depends on $c\cdot n$, i.e., the number of input coordinates times the number of bits used for their encoding. The width of this network is $O(1)$.

Simulating an entire layer requires iteratively simulating each neuron of the layer as described above, and then encoding the output of each neuron from the target network within a single number. In more details, the subnetwork $F_\ell$ iteratively simulates a single neuron from the $\ell$-th layer of the target network using the method described above. It also keeps track of the input (which encodes in a single number all the output coordinates from the previous layer) and a single designated output coordinate. After simulating a neuron, the network truncates the output to having only $c$ bits, and stores it in an output coordinate, where the output of the $i$-th neuron from the target network is stored in the $(i-1)\cdot c +1$ until $i\cdot c$ bits of this designated output coordinate. In total, the network $F_\ell:\reals\rightarrow\reals$ has an input dimension of $1$, i.e. its input is a single number representing an encoding of the $(\ell-1)$-th layer's outputs, and it outputs a single number with an encoding of the $\ell$-th layer. The width of this network is $O(1)$, and its depths depends on the $n_1\cdot n_2\cdot c$, where $n_1,n_2$ are the input and output dimensions of the $\ell$-th layer, and $c$ is the number of bits for each neuron.

\subsection{On the Number of Parameters in the Construction}\label{subsec:num param}

As we already discussed, the number of parameters in our construction is almost the same as the number of parameters in the target network. The main difference is that in our construction we have an extra $L$ term, and extra logarithmic terms in the other parameters of the problem. Here we will discuss why these extra terms come up in our construction, and in what situations they can be avoided.

The network we construct in \thmref{thm:wide to narrow} can be roughly represented as $\Ncal:=F_L\circ\cdots\circ F_1\circ F_{\text{enc}}$, that is, encoding the data and then simulating all the layers from the target network. Since we use bit extraction techniques for this construction, we cannot represent exactly the inputs and the weights of the target network. To this end, we only keep track of the $c$ most significant bits of each component of the target network (weights and inputs). 

The approximation capacity of our construction depends on the Lipschitz parameter of the each layer of the target network, and on $c$, the number of bits we store. To see this, first note that to approximate some number in $[0,1]$ up to an error of $\epsilon$, requires to store only its $O\left(\log\left(\epsilon^{-1}\right)\right)$ most significant bits. Recall that by our assumptions the weights of the target network are bounded in $[-B,B]$, and each coordinate of the input data is bounded in $[-A,A]$. The Lipschitz parameter of each layer of the network can be roughly upper bounded by $O\left(nB\right)$, where $n$ is the width of the target network. 
This means that after $L$ layers, the Lipschitz parameter of the network can be bounded by $(nB)^{O(L)}$. Using this estimate, it can be seen that to get an $\epsilon$ approximation of the output, storing $O(L\log(nAB\epsilon^{-1}))$ bits for every weight and input coordinate can suffice.


The number of parameters for simulating each layer of the target network depends on the number of stored bits, hence the number of parameters in our construction increases by logarithmic factors, and an $L$ factor. We get an $L^2$ term in the total number of parameters, since there are $L$ layers, and the simulation of each layer involves a blow-up by a factor of $L$.

We emphasize that this blow-up in the number of parameters is mainly due to a rough estimate of the Lipschitz constant for each layer of the target network. Since we use an efficient bit extraction technique, the number of parameters increases only by log of the Lipschitz constant. 

Informally, if the Lipschitz parameter of the network and its intermediate computations is small (which seems to often occur in practice, see for example \cite{fazlyab2019efficient, scaman2018lipschitz, latorre2020lipschitz}), then we believe that the extra $L$ factor can be reduced or even removed all together.
However, a formal statement requires a more delicate analysis, which we leave for future works.

\subsection{Extension to Multiple Outputs}
Our construction can be readily extended to the case where there are multiple outputs to the target network. Given some target network $\Ncal_0:[-A,A]^{d}\rightarrow\reals^{d_{\text{out}}}$, we use a similar construction to \thmref{thm:wide to narrow}, except for simulating the last layer. To simulate the last layer, given an encoding of the penultimate layer of $\Ncal_0$, we simulate each output in parallel in a similar manner as we did for a single output in \thmref{thm:wide to narrow}. In more details, for each output coordinate $i\in[d_{\text{out}}]$ we construct a subnetwork $F_i$ which given an encoding of the values from the penultimate layer, computes the $i$-th output. The construction of each $F_i$ is exactly the same as the construction from the proof of \thmref{thm:wide to narrow} which simulates the last layer of a target network with a single output. Now, the last layer computes 
\[
\bx\mapsto \begin{pmatrix}
F_1(\bx) \\ \vdots \\ F_{d_{\text{out}}}(\bx)
\end{pmatrix}~.
\]
Since the width of the subnetwork which simulates a layer is $O(1)$, the width of this new network increases by a factor of $O(d_{\text{out}})$, and the depth of this network does not change.

\subsection{Approximation With Bounded Weights}\label{subsec:bounded weights}
Our construction in \thmref{thm:wide to narrow} uses a network with very large weights (exponential in $L$ and $n$, see \thmref{thm:wide to narrow appen} in the appendix for the exact expression), which may be seen as a limitation of our construction. In this section we show that having such large weights can be easily avoided by slightly altering our construction from \thmref{thm:wide to narrow}. This change results in an extra linear factor, and some log factors on the number of parameters.


The reason we do have such large weights is because we use a bit extraction technique, which requires that 
the subnetworks in our construction
will have a very large Lipschitz constant. For example, constructing a neural network which outputs the $i$-th bit of its input (say, in $1$ dimension), requires that the Lipschitz constant will be approximately $2^i$.
For this reason, in several places in the proof our weights are exponential in the parameters of the problem, and they equal exactly to $2^N$ for some large $N$ which depends on the parameters of the problem. To avoid such a blow-up in the size of the weights, we can approximate a weight of size $2^N$ by just using $\log(N)$ layers, and multiplying $\log(N)$ times the number $2$ to obtain the same result. Using this technique we are able to construct a network with bounded weights but at the cost of increasing the number of parameters of the network, up to logarithmic terms, by a linear term in $n$ and $L$:

\begin{corollary}\label{cor:bounded weights}
Under the same setting as in \thmref{thm:wide to narrow}, there exists a neural network $\Ncal:[-A,A]^d\rightarrow\reals$ with  width bounded by $\max\{5d,10\}$, depth bounded by  $O\left(n^3L^3\log\left(ABn\epsilon^{-1}\delta^{-1}\beta\right)^2\right)$ and weights bounded by $2$, such that w.p $ > 1-\delta$ over $\bx\sim \Dcal$ we have that:
\[
|\Ncal(\bx) - \Ncal_0(\bx)| \leq \epsilon~.
\]
The total number of parameters in $\Ncal$ is $O\left(n^3L^3\log\left(ABn\beta\epsilon^{-1}\delta^{-1}\right)^2\right)$.
\end{corollary}

\begin{proof}
We use the homogeneity of the ReLU activation. Note that for a neuron of the form $\bz \mapsto \sigma(\inner{\bw,\bz}+b)$, we can divide the weights $\bw,b$ by some constant, and multiply the output of the neuron by the same constant, and for all $\bz$ the result will stay the same. Given the network constructed in \thmref{thm:wide to narrow}, denote its largest weight by $C$, and its depth by $L'$. Denote by $\tilde{W}^{(i)}$ and $\tilde{\bb}^{(i)}$ the weight matrices and biases of this network. We divide $\tilde{W}^{(1)}$ and $\tilde{\bb}^{(1)}$ by $C$. For each layer $i>1$, we divide $\tilde{W}^{(i)}$ by $C$, and $\tilde{\bb}^{(i)}$ by $C^i$. In the last layer we will multiply by $C^{L'}$.

We simulate the multiplication by $C^{L'}$ using small weights in the following way: We write $C^{L'} = 2^\alpha \cdot \beta$, where $\alpha \in \mathbb{N}$ with $\alpha \leq L'\log(C) + 1$ and $\beta \leq 2$. We note that the output may be negative, hence we need to simulate multiplication without the ReLU activation. To do that, we add a layer which acts as: $x\mapsto \begin{pmatrix}
\sigma(x) \\ \sigma(-x)
\end{pmatrix}$. We now use $\alpha$ layers to multiply each of the two outputs by the number $2$, and in the penultimate layer we multiply the result by $\beta$. The last layer acts as $\begin{pmatrix}
y_1 \\ y_2
\end{pmatrix} 
\mapsto y_1 - y_2$.
Note that since the second coordinate is equal to $\sigma(-y)$ and the first coordinate is equal to $\sigma(y)$ (for some $y$), then the output of the network is $y$. 

To prove the correctness of our construction, first note that 
the magnitude of 
each weight in our new network is bounded by $2$, since we divided each weight of the original network by $C^i$ for some $i\geq 1$ where $C$ is the size of the maximal weight. Second, we show that the output of the network is the same for all $\bx\in\reals^d$. Given some $\bx\in\reals^d$, denote by $\bx^{(i)}$ the output of the original network with input $\bx$ after $i$ layers. Assume by induction that after dividing the weights as explained above, the output of the $(i-1)$-th layer is divided by $C^{i-1}$, then for the $i$-th layer we have:
\[
\sigma\left(\frac{1}{C}\tilde{W}^{(i)}\cdot\frac{1}{C^{i-1}}\bx^{(i)} + \frac{1}{C^i}\tilde{\bb}^{(i)}\right) = \frac{1}{C^i} \sigma\left(\tilde{W}^{(i)}\cdot\bx^{(i)} + \tilde{\bb}^{(i)}\right)~.
\]
This means that after $L'$ layers, the output is divided by $C^{L'}$. Since we also multiply by this term in the last layers of the network, the output of the network does not change.

The weights of our construction are bounded by $2$. The width of our construction does not change from the width of the original network. The depth of our construction can be bounded by $O(L' + L'\log(C)) = O(L'\log(C))$ where $L'$ is the depth of the original network, and $C$ is the maximal weight in the original network. The log of the largest weight can be bounded by (see \thmref{thm:wide to narrow appen} in the appendix):

\begin{align*}\label{eq:bounded weights blow-up}
O\left(6Ln + \log\left(d^2nAB\beta\epsilon^{-1}\delta^{-1}\right) + d\log(2A)\right) = O\left(Ln\log\left(nAB\beta\epsilon^{-1}\delta^{-1}\right)\right)    
\end{align*}
Hence, the depth of the network can be bounded by $O\left(n^3L^3\log\left(ABn\beta\epsilon^{-1}\delta^{-1}\right)^2\right)$.
The number of parameters in the network also increases by $O(L'\log(C))$. Hence, the total number of parameters in the network can be bounded by $O\left(n^3L^3\log\left(ABn\beta\epsilon^{-1}\delta^{-1}\right)^2\right)$.
\end{proof}

\corollaryref{cor:bounded weights} shows that even if we use networks with constant weights, we can simulate any target network up to any accuracy using a deep and narrow network, while having only a polynomial blow-up in the parameters of the problem. Moreover, 
the number of parameters in this construction is only larger by a factor of $\tilde{O}(nL)$ than the construction in \thmref{thm:wide to narrow}. 
An interesting question is whether a better bound can be achieved using a different construction. We leave this question for future research.

\begin{remark}
Instead of bounding the magnitude of the weights in the network, we could have bounded the bit complexity of the network. By bit complexity, we mean the number of bits that are needed to represent all the weights of the network. By carefully following the proof of \thmref{thm:wide to narrow}, it can be seen that each weight in our construction can be represented by at most $\tilde{O}(n\cdot L)$ bits. We note that although it seems possible to provide a construction where each weight can be represented with $O(1)$ bits, at the cost of increasing the number of parameters in the network (by similar arguments to the proof of \corollaryref{cor:bounded weights}), such a construction does not seem to reduce the overall bit complexity of the network. This is because, we still use the same total number of bits to represent all the weights of the network, but we spread those bits across more weights. An interesting question is whether it is possible to provide a construction with smaller bit complexity, and we leave it for future work.
\end{remark}

\section{Achieving Close to Minimal Width}\label{sec:opt width}

Previous works have shown that neural networks over a compact input domain with width $< d$ (where $d$ is the input dimension) are not universal approximators, in the sense that they cannot approximate any function w.r.t. the $L^1$ norm up to arbitrarily small accuracy (see e.g. \cite{lu2017expressive,park2020provable}). Hence, we cannot expect to approximate any wide network using a narrow network with width less than $d$.\footnote{Note that the approximation in \thmref{thm:wide to narrow} is given w.h.p over some distribution, and not in the $L^1$ sense. However,  any construction that achieves approximation w.p. $>1 -  \delta$ can be used to obtain $L^1$ approximation, by bounding the output and choosing an appropriate $\delta$.} In this section we show how to approximate any wide network using a narrow network with width $d+2$, which is only larger than the lower bound by $2$. We note that in \cite{park2020minimum}, both an upper and lower bound of $d+1$ is shown for universal approximation over an unbounded domain, although their construction uses an exponential number of parameters. Our main result in this section is the following:


\begin{theorem}\label{thm:optimal width}
Assume the same setting as in \thmref{thm:wide to narrow}. Then, there exists a neural network $\Ncal:[-A,A]^d\rightarrow\reals$ with width $\max\{d+2,10\}$, and depth $O\left(n^2L^2\log(ABN\epsilon^{-1})\right)$, such that $w.p > 1-\delta$ over $\bx\sim \Dcal$ we have that:
\[
|\Ncal(\bx) - \Ncal_0(\bx)| \leq \epsilon~.
\]
The total number of parameters in $\Ncal$ is $O\left(dn^2L^2\log(ABn\epsilon^{-1})^2\right)$. If $n \geq d^{1.5}$, then total number of parameters is $O\left(n^2L^2\log(ABn\epsilon^{-1})^2\right)$.
\end{theorem}

The full proof can be found in \appref{appen:proofs from opt width}. The proof is very similar to the proof of \thmref{thm:wide to narrow}. The only difference is that we replace the first component of the network which encodes the input data. The new encoding scheme is more efficient in terms of width as it allows encoding the inputs coordinates using width $d+2$ instead of width $5d$. We extract the bits of each coordinate sequentially, instead of in parallel. This results in a blow-up on the number of parameters by a factor of $d$. The bit extraction technique we use here also relies on Telgarsky's triangle function, but to extract $c$ bits it requires a depth of $O(c^2)$, instead of a depth of $O(c)$ as in the proof of \thmref{thm:wide to narrow}. This results in a blow-up by a logarithmic factor on the number of parameters. 

We note that in  \cite{park2020minimum} the authors achieved a universal approximation result using width $d+1$, vs. width $d+2$ in our theorem. Moreover, they use a bit extraction technique somewhat reminiscent of ours. However, their required depth is exponential in the problem's parameters, while in our construction it is polynomial.
We conjecture that it is not possible to achieve a similar construction to ours with width less than $d+2$, unless we increase the number of parameters by a factor which is polynomial in both $n$, and the Lipschitz constant of the network. We leave this question for future research.

\begin{remark}
In \thmref{thm:optimal width} the number of parameters in the network increases by a factor of $d$ compared to the number of parameters in \thmref{thm:wide to narrow}. We note that this is due to the way we count the parameters of the network. We defined the number of parameters as the number of coordinates in its weight matrices and bias vectors. In \thmref{thm:optimal width}, the width of the subnetwork that encodes the input is $O(d)$ and its depth is $\tilde{O}(d)$, hence its number of parameters is $\tilde{O}(d^3)$. 
An alternative way to define the number of parameters is as the number of non-zero weights in the network. This alternative definition is used in many previous works (see, e.g., \cite{bartlett2019nearly,vardi2021optimal}).
Then, the number of parameters for the subnetwork which encodes the input is only $\tilde{O}(d^2)$, which gives us the exact same bound as in \thmref{thm:wide to narrow}.
\end{remark}

\section{Exact Representation With Deep and Narrow Networks}\label{sec:exact}
In \secref{sec:narrow to deep} we showed a construction for approximating a target shallow and wide neural network using a deep and narrow neural network. We note that this construction assumes that the data is bounded in $[-A,A]^d$ and that we approximate the target network w.h.p over some distribution up to an error of $\epsilon$. The number of parameters in the construction depends logarithmically on $A$ and $\epsilon$.

In this section we show a different construction which exactly represents the target network for all $\bx\in\reals^d$ using a deep and narrow construction. We will also discuss in which cases this construction is better than the one given in \thmref{thm:wide to narrow}. We show the following:

\begin{theorem}\label{thm: exact}
Let $\Ncal^*:\reals^d\rightarrow\reals$ be a neural network with $L$ layers and width $n$. Then, there exists a neural network $\Ncal:\reals^d\rightarrow\reals$ with width $2(d+L-1)$ and depth $(2n)^{L-1} + 2$ such that for every $\bx\in\reals^d$ we have that $\Ncal(\bx) = \Ncal^*(\bx)$. 
\end{theorem}

The full proof can be found in \appref{appen:proofs from sec exact}, but in a nutshell, is based on an inductive argument over the layers of the network (starting from the bottom layer and ending in the output neuron). Specifically, fix some layer, and consider some neuron $i$ in that layer (where $i$ ranges from $1$ to the width of that layer). We can view the output of that neuron as the output of a subnetwork $\Ncal_i:\reals^d\rightarrow\reals$ which ends at that neuron. Suppose by induction that we can convert this subnetwork $\Ncal_i$ to an equivalent subnetwork which is narrow. Doing this for all $i$, we get a sequence of narrow subnetworks $\Ncal_1,\Ncal_2,\ldots$ which represent the outputs of all neurons in the layer. Now, instead of placing them side-by-side (which would result in a wide network), we put them one \emph{after} the other, using in parallel $O(d)$ neurons to remember the original inputs, and another $O(1)$ neurons per layer to incrementally accumulate a weighted linear combination of the subnetworks' outputs, mimicking the computation of the layer at the original network. Overall, we end up with a narrow network which mimics the outputs of the original layer, which we can then use inductively for constructing the outputs of the following layers.

We emphasize that this construction is not an approximation of the target network, but an exact representation of it using a deep and narrow network. Note that our construction is narrow only if $L<<n$, otherwise the target network might be narrower than our construction. This construction is not efficient in the sense that we compute each neuron many times. For example, a neuron in the first layer of the target network is computed exponentially many times (in $L$). This is because, every neuron in a consecutive layer computes this neuron recursively. 


\subsubsection*{Cases Where the Exact Representation is Efficient}
We argue that for $d=O(1)$ and $L=2,3$, the construction presented in this section does not significantly increase the number of parameters compared to the target network.
The construction in \thmref{thm: exact} has width $O(d+L)$ throughout the entire network. 
Also, the depth of the network constructed in \thmref{thm: exact} is exponential in $L$. For these reasons, the number of parameters in the network constructed in \thmref{thm: exact} is $O\left((d+L)^2\cdot n^{L-1}\right)$, which seems less efficient than the construction in \thmref{thm:wide to narrow}. 

Assume that the input dimension is constant, that is $d=O(1)$, and we are only interested in the asymptotic dependence on $n$ for different values of $L$. 
If $L=2$, then the target network has $O(n)$ parameters, because it is a depth-$2$ network with constant input dimension. By the bound we saw above, the construction in \thmref{thm: exact} also have $O(n)$ parameters. 
If $L=3$, then the target network has $O(n^2)$ parameters, and the construction presented in this section also has $O(n^2)$ parameters. For $L\geq 4$, since the target network has $O\left(L\cdot n^2\right)$ parameters, while the  construction from \thmref{thm: exact} has $O\left(L^2\cdot n^{L-1}\right)$ parameters, then the construction does increase the number of parameters.

We emphasize that the construction here simulates a wide network using a deep network with width independent of $n$ (the width of the target network), and that it is an exact representation for every $\bx\in\reals^d$. On the other hand, in \thmref{thm:wide to narrow} the construction only approximates the target network up to some $\epsilon$, with high probability and in a bounded domain.
We conjecture that it is not possible to obtain an exact representation without increasing the number of parameters for general $L$ and $d$. 

\section{Discussion}
In this work we solved an open question from \cite{lu2017expressive}. We proved that any target network with width $n$, depth $L$ and inputs in $\reals^d$ can be approximated by a network with width $O(d)$, where the number of parameters increases by only a factor of $L$ over the target network (up to log factors). Relying on previous results on depth separation (e.g. \cite{eldan2016power,safran2017depth,telgarsky2016benefits,daniely2017depth}), this shows that depth plays a more significant role in the expressive power of neural networks than width. We also extend our construction to having bounded weights, and having width at most $d+2$, where  previous lower bounds showed that such a construction is not possible for width less than $d$. Both of these extensions cause an extra polynomial blow-up in the number of parameters. Finally, we show a different construction which allows exact representation of wide networks using deep and shallow networks. We argue that this construction does not increase the number of parameters by more than constant factors when $d=O(1)$ and $L=2,3$.

There are a couple of future research directions which may be interesting to pursue. First, it would be interesting to see if the upper bound established in \thmref{thm:wide to narrow} is tight. Namely, whether the extra blow-up by a factor of $L$ and by logarithmic factors is unavoidable. Second, it would be interesting to find a more efficient construction than in \thmref{thm: exact} for exact representation of wide networks using narrow networks, or to establish a lower bound which shows that it is not possible. Finally, in terms of optimization, given two approximations of the same function, one using a narrow and deep network, and the other using a shallow and wide network, it would be interesting to analyze their optimization process, and see which representation is easier to learn using standard methods (e.g. SGD).

\subsubsection*{Acknowledgments}
This research is supported by the European Research Council (ERC) grant 754705. 

\setcitestyle{numbers}
\bibliographystyle{abbrvnat}
\bibliography{mybib}

\appendix
\section{Proofs from \secref{sec:narrow to deep}}\label{appen:proof of main theorem}

\subsection{Encoding of the data}

\begin{lemma}\label{lem:efficient encoding parameter}
Let $\delta>0,~c,c_0,A\in\naturals$, where $c\geq c_0 +  \log(A)  + 1$. There exists a neural network $\Ncal:\reals^d\rightarrow\reals$ with width $5d$ depth at most $O(c_0)$ and weights bounded by $O\left(\frac{2^{(d-1)c}\cdot d}{\delta}\right)$, such that if we sample $\bx\sim U([0,A]^d)$, then w.p $>1-\delta$ for every $i\in[d]$ we have that:
\begin{align*}
    &\bin_{(i-1)\cdot c + 1:i\cdot c} \left(\Ncal(\bx)\right) = \left\lfloor\frac{x_i}{A}\cdot 2^{c_0}\right\rfloor\cdot A 
\end{align*}
\end{lemma}

\begin{proof}

We first use \lemref{lem:enc data single coord} to construct a network $G:\reals\rightarrow\reals$, such that w.p at least $1-\frac{\delta}{d}$, if we sample $x\sim U([0,1])$ then $G(x) = \lfloor x\cdot 2^{c_0}\rfloor$. Also, $G(x)$ has width 5 and depth bounded by $O(c_0)$. We define a network $\Ncal:\reals^{d}\rightarrow\reals$ which maps the following input to output:
\[
\begin{pmatrix}
x_1 \\ \vdots \\ x_d
\end{pmatrix} \mapsto \begin{pmatrix}
\sum_{i=1}^{d} 2^{(i-1)\cdot c} A\cdot G\left(\frac{x_i}{A}\right)
\end{pmatrix}
\]
We can construct $\Ncal$ such that it has width $5d$ and depth bounded by $O(c_0)$ in the following way: We first map:
\[
\begin{pmatrix}
x_1 \\ \vdots \\ x_d
\end{pmatrix} \mapsto \begin{pmatrix}
\frac{x_1}{A} \\ \vdots \\ \frac{x_d}{A}
\end{pmatrix} \mapsto \begin{pmatrix}
A\cdot G\left(\frac{x_1}{A}\right) \\ \vdots \\ A\cdot G\left(\frac{x_d}{A}\right)
\end{pmatrix}
\]
This can be done using width $5d$ and depth $O(c_0)$, since calculating each $G(x)$ requires a width of $5$ and depth $O(c_0)$. In the last layer of $\Ncal$ we sum all the $A\cdot G\left(\frac{x_i}{A}\right)$ with the corresponding weights.

We note that if $x\sim U([0,A])$, then $\frac{x}{A} \sim U([0,1])$. Hence, by \lemref{lem:enc data single coord} and the union bound, the output is correct for all $i\in[d]$ w.p $ > 1-\delta$. Hence, by our construction, $\Ncal$ satisfies the conditions of the lemma.

The maximal width of $\Ncal$ is the maximal width of its subnetworks which is $5d$. The depth of $\Ncal$ is the sum of the depths of its subnetworks which can be bounded by $O(c_0)$. The maximal weight of $\Ncal$ can be bounded by the maximal weight of its subnetworks. The maximal weight of $\Ncal$ can be bounded by $O\left(\frac{2^{(d-1)c}\cdot d}{\delta}\right)$, which appears in its last layer.
\end{proof}

\begin{lemma}\label{lem:enc data single coord}
Let $\delta > 0$ and $c\in\naturals$. There exists a neural network $\Ncal:\reals\rightarrow\reals$ with width 5, depth bounded by $O(c)$ and weights bounded by $O\left(\frac{2^c}{\delta}\right)$, such that if we sample $x\sim U\left([0,1]\right)$, w.p $> 1-\delta$ we have that $\Ncal(x)= \left\lfloor x\cdot 2^{c} \right\rfloor$.
\end{lemma}

\begin{proof}
We define $\varphi(z) = \sigma(\sigma(2z)-\sigma(4z-2))$, this is Telgarsky's triangle function \cite{telgarsky2016benefits}. We also define the following function for $i\in[c]$:

\begin{align}\label{eq:telgarki minus telgarski}
    \psi_i(x) = \frac{2^{c+2-i}}{\delta}\sigma\left(  \varphi^{(i)}\left(x + \frac{\delta}{2^{c+2}} \right) - \varphi^{(i)}\left(x + \frac{\delta}{2^{c+1}}\right)\right) ~.
\end{align}
The intuition behind \eqref{eq:telgarki minus telgarski} is the following: The function $\varphi^{(i)}$  is a piecewise linear function with $2^{i-1}$ "bumps". Each such "bump" consists of two linear parts with a slope of $2^{i}$, the first linear part goes from 0 to 1, and the second goes from 1 to 0. Let $x\in[0,1]$, it can be seen that the $i$-th bit of $x$ is 1 if $\varphi^{(i)}\left(x\right)$ is on the second linear part (i.e. descending from 1 to 0) and its $i$-th bit is 0 otherwise. 

Assume that $x,~x + \frac{\delta}{2^{c+1}}$ and $x + \frac{\delta}{2^{c+2}}$ are on the same linear piece of $\varphi^{(i)}$ for $i\leq  c$. Then, the $i$-th bit of $x$ is equal to 1 if $\varphi^{(i)}\left(x + \frac{\delta}{2^{c+1}} \right) - \varphi^{(i)}\left(x + \frac{\delta}{2^{c+2}}\right)  > 0$, and $0$ otherwise. Also, if we sample $x\sim U([0,1])$, then w.p $<\delta$ both terms are on different linear pieces for some $i$.

Using this observation we get that the output of $\psi_i(x)$ is equal to the $i$-th bit of $x$ w.p $>1-\delta$ over sampling $x\sim U([0,1])$.

We construct a network $f_i(x):\reals^3\rightarrow\reals^3$ which maps the following input to output:
\[
\begin{pmatrix}
y \\ \varphi^{(i-1)}\left(x + \frac{\delta}{2^{c+1}} \right) \\ \varphi^{(i-1)}\left(x + \frac{\delta}{2^{c+2}} \right) 
\end{pmatrix} \mapsto \begin{pmatrix}
2y + \psi_i(x) \\ \varphi^{(i)}\left(x + \frac{\delta}{2^{c+1}} \right) \\ \varphi^{(i)}\left(x + \frac{\delta}{2^{c+2}} \right)
\end{pmatrix}
\]
This network can be realized using four layers (two for Telgarsky's function, one for $\psi_i$ and one for the output) and width $5$ (one for storing $y$ and four for applying Telgarsky's function twice). We also define $f_0:\reals\rightarrow\reals^3$ as:
\[
f_0(x) = \begin{pmatrix}
0 \\ x + \frac{\delta}{2^{c+1}} \\ x + \frac{\delta}{2^{c+2}}
\end{pmatrix}
\]
Finally, we construct the network:
\[
\Ncal := P_1 \circ f_c \circ \cdots \circ f_1 \circ f_0~,
\]
where $P_1$ is the projection on the first coordinate. By our argument above, w.p $>1-\delta$ over sampling $x\sim U([0,1])$ we get that $\Ncal(x) = \left\lfloor x\cdot 2^{c} \right\rfloor$ as required. The width of $\Ncal$ is the maximal width of each of its subnetworks which is at most $5$. The depth of $\Ncal$ is the sum of the depths of its subnetworks, which can be bounded by $O(c)$. Each weight of $\Ncal$ can be bounded by $O\left(\frac{2^c}{\delta}\right)$
\end{proof}

\subsection{Approximation of a single neuron}
In the following we show that given a previous layer with $n$ neurons each encoded with $c$ bits, we can construct a network which outputs a single neuron defined by some given weights.

\begin{lemma}\label{lem:single neuron}
Let $c,n\in\naturals$, let $b,w_1,\dots,w_n\in\naturals$ with $\len(b),\len(w_i) \leq c$ for every $i\in[n]$ and let $\alpha_1,\dots,\alpha_n\in\{\pm 1\}$. There exists a neural network $\Ncal:\reals\rightarrow\reals$ with width $8$, depth at most $O(n\cdot c)$ and weights bounded by $O\left(2^{n\cdot c}\right)$, such that for every $x\in\naturals$ with $\len(x) \leq n\cdot c$ we have that:
\[
\Ncal(x) = \sigma\left(\sum_{i=1}^{n-1} \alpha_i w_i\bin_{(i-1)\cdot c + 1: i\cdot c}(x) +b \right)~.
\]
\end{lemma}

\begin{proof}
We construct two sets of networks: $F_i$ for $i\in[n]$ and $f_{i,j}$ for $i\in[n],~j\in[c]$. The intuition is that each $f_{i,j}$ will decode the $j$-th bit from the $i$-th input neuron, and $F_i$ will add up the $i$-th input neuron to the output neuron.

The construction of the bit extraction is similar to the one from \eqref{eq:telgarki minus telgarski}. We first define $\varphi(z) = \sigma(\sigma(2z) - \sigma(4z-2))$ which is Telgarsky's triangle function. We also define the following function for every $\ell\in[n\cdot c]$:
\begin{equation}\label{eq:telgarsky minus telgarsky for naturals}
    \psi_\ell(x) = 2^{n\cdot c +2 - \ell}\sigma\left(\varphi^{(\ell)}\left(\frac{x}{2^{n\cdot c}} + \frac{1}{2^{n\cdot c + 2}}\right) - \varphi^{(\ell)}\left(\frac{x}{2^{n\cdot c}} + \frac{1}{2^{n\cdot c + 1}}\right) \right)~.
\end{equation}
By the same reasoning as in \eqref{eq:telgarki minus telgarski}, the output of $\psi_\ell(x)$ is equal to the $\ell$-th bit of $x$, for every $x\in\naturals$ with $\len(x) \leq n\cdot c$.

Let $i\in[n]$ and $j\in[c]$, then we define $f_{i,j}:\reals^6\rightarrow\reals^6$ which maps the following input to output:
\[
\begin{pmatrix}
x \\  x_{\text{cur}} \\ y_{\text{pos}} \\ y_{\text{neg}} \\ \varphi^{\left((i-1)\cdot c + j - 1\right)}\left(\frac{x}{2^{n\cdot c}} + \frac{1}{2^{n\cdot c + 1}}\right)  \\ \varphi^{\left((i-1)\cdot c + j - 1\right)}\left(\frac{x}{2^{n\cdot c}} + \frac{1}{2^{n\cdot c + 2}}\right) 
\end{pmatrix} \mapsto
\begin{pmatrix}
x \\  2\cdot x_{\text{cur}} + \psi_{(i-1)\cdot c + j }(x) \\ y_{\text{pos}} \\ y_{\text{neg}} \\ \varphi^{\left((i-1)\cdot c + j \right)}\left(\frac{x}{2^{n\cdot c}} + \frac{1}{2^{n\cdot c + 1}}\right)  \\ \varphi^{\left((i-1)\cdot c + j \right)}\left(\frac{x}{2^{n\cdot c}} + \frac{1}{2^{n\cdot c + 2}}\right) 
\end{pmatrix}
\]
where we calculate $\psi_{(i-1)\cdot c + j }(x)$ using \eqref{eq:telgarsky minus telgarsky for naturals}. 

For $i\in[n]$ we define $F_i:\reals^6\rightarrow\reals^6$ which maps the following input to output if $\alpha_i = +1$:
\begin{equation}\label{eq:single neuron positive alpha i}
\begin{pmatrix}
x \\  x_{\text{cur}} \\ y_{\text{pos}} \\ y_{\text{neg}} \\ \varphi^{\left(i\cdot c\right)}\left(\frac{x}{2^{n\cdot c}} + \frac{1}{2^{n\cdot c + 1}}\right)  \\ \varphi^{\left(i\cdot c\right)}\left(\frac{x}{2^{n\cdot c}} + \frac{1}{2^{n\cdot c + 2}}\right) 
\end{pmatrix} \mapsto \begin{pmatrix}
x \\  0 \\ y_{\text{pos}} + w_i\cdot x_{\text{cur}} \\ y_{\text{neg}} \\ \varphi^{\left(i\cdot c\right)}\left(\frac{x}{2^{n\cdot c}} + \frac{1}{2^{n\cdot c + 1}}\right)  \\ \varphi^{\left(i\cdot c\right)}\left(\frac{x}{2^{n\cdot c}} + \frac{1}{2^{n\cdot c + 2}}\right) 
\end{pmatrix}~,
\end{equation}

and if $\alpha_i = -1$:
\begin{equation}\label{eq:single neuron negative alpha i}
\begin{pmatrix}
x \\  x_{\text{cur}} \\ y_{\text{pos}} \\ y_{\text{neg}} \\ \varphi^{\left(i\cdot c\right)}\left(\frac{x}{2^{n\cdot c}} + \frac{1}{2^{n\cdot c + 1}}\right)  \\ \varphi^{\left(i\cdot c\right)}\left(\frac{x}{2^{n\cdot c}} + \frac{1}{2^{n\cdot c + 2}}\right) 
\end{pmatrix} \mapsto \begin{pmatrix}
x \\  0 \\ y_{\text{pos}}  \\ y_{\text{neg}} + w_i\cdot x_{\text{cur}} \\ \varphi^{\left(i\cdot c\right)}\left(\frac{x}{2^{n\cdot c}} + \frac{1}{2^{n\cdot c + 1}}\right)  \\ \varphi^{\left(i\cdot c\right)}\left(\frac{x}{2^{n\cdot c}} + \frac{1}{2^{n\cdot c + 2}}\right) 
\end{pmatrix}~.
\end{equation}

We now define $G_i:\reals^6\rightarrow\reals^6$ for $i\in[n]$ as:
\[
G_i := F_i \circ f_{i,c}\circ \cdots\circ f_{i,1}~.
\]
In words, the goal of each $G_i$ is to add to the output neuron the output of the $i$-th input neuron multiplied by its corresponding weight. We also define the input and output networks $G_{\text{in}}:\reals\rightarrow\reals^6$, $G_{\text{out}}:\reals^6\rightarrow\reals$ as:
\begin{align*}
    &G_{\text{in}}\left(x\right) = \begin{pmatrix}
    x \\ 0 \\ 0 \\ 0 \\ \frac{x}{2^{n\cdot c}} +\frac{1}{2^{n\cdot c + 1}} \\ \frac{x}{2^{n\cdot c}} + \frac{1}{2^{n\cdot c + 2}}
    \end{pmatrix} \\
    &G_{\text{out}}\left(\begin{pmatrix}
    x \\  x_{\text{cur}} \\ y_{\text{pos}} \\ y_{\text{neg}} \\ z_1 \\ z_2 
    \end{pmatrix}\right) = \sigma(y_{\text{pos}} - y_{\text{neg}} + b)
\end{align*}
Finally, we define the network $\Ncal:\reals\rightarrow\reals$ as:
\[
\Ncal := G_{\text{out}} \circ G_n\circ \cdots \circ G_1 \circ G_{\text{in}}~.
\]
By the construction of $\Ncal$, and each of the $G_i$ we get that for every $x\in\naturals$ with $\len(x)\leq n\cdot c$:

\[
\Ncal(x) = \sigma\left(\sum_{i=1}^{n-1} \alpha_i w_i\bin_{(i-1)\cdot c + 1: i\cdot c}(x) +b \right)~.
\]
The width of each $f_{i,j}$ can be bounded by $8$. This is because each $\varphi$ requires a width of $2$, and simulating the identity requires a width of 1, since all the inputs are positive (hence $\sigma(x)=x$). The width of each $F_i$ is $6$, since it only requires addition and simulating the identity on positive inputs. The width of $G_{\text{in}}$ and $G_{\text{out}}$ can also be bounded by $6$. Hence, the width of $\Ncal$ is at most $8$. The depth of each $f_{i,j}$ can be bounded by $4$, and the depth of each $F_i$ and of $G_{\text{in}}$ and $G_{\text{out}}$ can be bounded by $2$. In total, the depth of $\Ncal$, which is bounded by the sum of depths of its subnetworks, can be bounded by $O(n\cdot c)$. The maximal weight of each $f_{i,j}$ is $O\left(2^{n\cdot c}\right)$. The maximal weight of each $F_i$ can be bounded by $\max_i |w_i| \leq \log(c)$. In total, the weights of $\Ncal$ can be bounded by $O\left(2^{n\cdot c}\right)$.
\end{proof}

\subsection{Approximation of a layer}
In the following we show that given a previous layer with $n_1$ neurons, each encoded with $c$ bits, we can construct a network which outputs an encoded output layer with $n_2$ neurons.

\begin{lemma}\label{lem: approx a layer}
Let $c,n_1,n_2\in\naturals$. For every $j\in[n_1],~i\in[n_2]$ let $w_{i,j}\in\naturals$ with $\len(w_{i,j})\leq c$ and $\alpha_{i,j}\in\{\pm 1\}$. For every $j\in[n_2]$ let $b_i\in\naturals$ with $\len(b_j)\leq c$. 
Then, there exists a neural network $\Ncal:\reals\rightarrow\reals$ with width $10$, depth bounded by $O(n_1n_2c)$, and weights bounded by $O\left(2^{n_1c}\right)$ with the following property: Let $x\in\naturals$ with $\len(x)\leq c\cdot n_1$ such that for every $i\in[n_2]$ we have: 

\begin{equation}\label{eq:len bound x}
\len\left( \sum_{j=1}^{n_1}\alpha_{i,j}w_{i,j}\bin_{(j-1)\cdot c + 1:j\cdot c}(x) + b_i\right) \leq c~.
\end{equation}
Then, for every $i\in[n_2]$ we get: 
\begin{equation}\label{eq:bin of Ncal layer}
\bin_{(i-1)\cdot c + 1:i\cdot c}(\Ncal(x)) = \sigma\left(\sum_{j=1}^{n_1}\alpha_{i,j}w_{i,j}\bin_{(j-1)\cdot c + 1:j\cdot c}(x) + b_i\right)~.
\end{equation}
\end{lemma}

\begin{proof}
For every $i\in[n_2]$ we use \lemref{lem:single neuron} to construct a network $\tilde{F}_i$ such that for every $x\in\naturals$ with $\len(x)\leq n_1\cdot c$ we get:
\[
\tilde{F}_i(x) = \sigma\left(\sum_{j=1}^{n_1} \alpha_{i,j} w_{i,j}\bin_{(j-1)\cdot c + 1: j\cdot c}(x) +b_i \right)~.
\]
We now construct a network $F_i:\reals^2\rightarrow\reals^2$ such that:
\[
F_i\left(\begin{pmatrix}
x \\ y
\end{pmatrix}\right) = \begin{pmatrix}
x \\ y\cdot 2^c + \tilde{F}_i(x)~.
\end{pmatrix}
\]
We also construct the input and output subnetworks $G_{\text{in}}:\reals\rightarrow\reals^2,~ G_{\text{out}}:\reals^2\rightarrow\reals$ as:
\begin{align*}
    & G_{\text{in}}(x) = \begin{pmatrix}
    x \\ 0
    \end{pmatrix} \\
    & G_{\text{out}}\left(\begin{pmatrix}
    x \\ y
    \end{pmatrix}\right) = y~.
\end{align*}
Finally, we construct the network $\Ncal:\reals\rightarrow\reals$ as:
\[
\Ncal := G_{\text{out}} \circ F_{n_2}\circ \cdots \circ F_1 \circ G_{\text{in}}~.
\]

Let $x$ which satisfied \eqref{eq:len bound x} for every $i\in[n_2]$. This means that for every $i\in[n_2]$ we have that $\len(\tilde{F}_i(x))\leq c$. From this we get that \eqref{eq:bin of Ncal layer} is satisfied. 

The width of each $\tilde{F}_i$ is at most $8$ by \lemref{lem:single neuron}. The width of each $F_i$ is at most $10$, since we simulate $\tilde{F}_i$ as well as the identity twice (for $x$ and $y$). Note that since both $x$ and $y$ are non-negative, then a single neuron can simulate the identity ($\sigma(z)=z$ for all $z\geq 0$). The width of $\Ncal$ is the maximal width of its subnetworks, which by the above calculation is $10$. The depth of $\Ncal$ is the sum of the depths of its subnetworks. The depth of each $\tilde{F}_i$ is bounded by $O(n_1\cdot c)$, hence this also bounds the depth of each $F_i$. The depth of $G_{\text{in}}$ and $G_{\text{out}}$ are $1$. Hence the depth of $\Ncal$ can be bounded by $O(n_1n_2c)$. The maximal weight of $\Ncal$ can be bounded by the maximal weight of each $F_i$, which by \lemref{lem:single neuron} is $O\left(2^{n_1\cdot c}\right)$.
\end{proof}

We will also use the following lemma to correctly scale the number of bits stored after each layer. 

\begin{lemma}\label{lem:compress x}
Let $a,a_0,n\in\naturals$ with $a_0<  a$. There exists a neural network $\Ncal:\reals\rightarrow\reals$ with width $7$, depth $O(n\cdot a)$ and weights bounded by $O\left(2^{n\cdot a}\right)$ with the following property: Let $x\in\naturals$ with $\len(x)\leq a\cdot n$, then for every $j\in[n]$ we have:
\[
\bin_{(j-1)\cdot a + 1:j\cdot a}(\Ncal(x)) = \left\lfloor\bin_{(j-1)\cdot a + 1:j\cdot a}(x) \cdot 2^{-a_0}\right\rfloor
\]
\end{lemma}

\begin{proof}
The construction is similar to that of \lemref{lem:single neuron}, where we recursively decode each part of the input. We define $\varphi(z)= \sigma(\sigma(2z) - \sigma(4z-2))$, and for any $\ell\in [n\cdot a]$:
\begin{equation}\label{eq: psi ell with a}
    \psi_\ell(x) = 2^{n\cdot a +2 - \ell}\sigma\left(\varphi^{(\ell)}\left(\frac{x}{2^{n\cdot a}} + \frac{1}{2^{n\cdot a + 2}}\right) - \varphi^{(\ell)}\left(\frac{x}{2^{n\cdot a}} + \frac{1}{2^{n\cdot a + 1}}\right) \right)~.
\end{equation}
Note that the output of $\psi_\ell(x)$ is equal to the $\ell$-th bit of $x$ for every $x\in\naturals$ with $\len(x)\leq n\cdot a$ (see the explanation after \eqref{eq:telgarki minus telgarski}).

Let $i\in[n]$ and $j\in[a-a_0]$, then we define $f_{i,j}:\reals^6\rightarrow\reals^6$ which maps the following input to output:
\[
\begin{pmatrix}
x \\  x_{\text{cur}} \\ y \\ \varphi^{\left((i-1)\cdot a + j - 1\right)}\left(\frac{x}{2^{n\cdot a}} + \frac{1}{2^{n\cdot a + 1}}\right)  \\ \varphi^{\left((i-1)\cdot a + j - 1\right)}\left(\frac{x}{2^{n\cdot a}} + \frac{1}{2^{n\cdot a + 2}}\right) 
\end{pmatrix} \mapsto
\begin{pmatrix}
x \\  2\cdot x_{\text{cur}} + \psi_{(i-1)\cdot a + j }(x)  \\ y \\ \varphi^{\left((i-1)\cdot a + j \right)}\left(\frac{x}{2^{n\cdot a}} + \frac{1}{2^{n\cdot a + 1}}\right)  \\ \varphi^{\left((i-1)\cdot a + j \right)}\left(\frac{x}{2^{n\cdot a}} + \frac{1}{2^{n\cdot a + 2}}\right) 
\end{pmatrix}
\]
where we calculate $\psi_{(i-1)\cdot a + j }(x)$ using \eqref{eq: psi ell with a}. For every $j\in\{a- a_0+1,\dots , a\}$ we define $f_{i,j}:\reals^6\rightarrow\reals^6$ which maps the following input to output:
\[
\begin{pmatrix}
x \\  x_{\text{cur}} \\ y \\ \varphi^{\left((i-1)\cdot a + j - 1\right)}\left(\frac{x}{2^{n\cdot a}} + \frac{1}{2^{n\cdot a + 1}}\right)  \\ \varphi^{\left((i-1)\cdot a + j - 1\right)}\left(\frac{x}{2^{n\cdot a}} + \frac{1}{2^{n\cdot a + 2}}\right) 
\end{pmatrix} \mapsto
\begin{pmatrix}
x \\  x_{\text{cur}} \\ y \\ \varphi^{\left((i-1)\cdot a + j \right)}\left(\frac{x}{2^{n\cdot a}} + \frac{1}{2^{n\cdot a + 1}}\right)  \\ \varphi^{\left((i-1)\cdot a + j \right)}\left(\frac{x}{2^{n\cdot a}} + \frac{1}{2^{n\cdot a + 2}}\right) 
\end{pmatrix}
\]

For $i\in[n]$ we define $F_i:\reals^5\rightarrow\reals^5$ which maps the following input to output:
\begin{equation*}
\begin{pmatrix}
x \\  x_{\text{cur}} \\ y \\ \varphi^{\left(i\cdot a\right)}\left(\frac{x}{2^{n\cdot a}} + \frac{1}{2^{n\cdot a + 1}}\right)  \\ \varphi^{\left(i\cdot a\right)}\left(\frac{x}{2^{n\cdot a}} + \frac{1}{2^{n\cdot a + 2}}\right) 
\end{pmatrix} \mapsto \begin{pmatrix}
x \\  0 \\ 2^{a}\cdot y + x_{\text{cur}}  \\ \varphi^{\left(i\cdot a\right)}\left(\frac{x}{2^{n\cdot a}} + \frac{1}{2^{n\cdot a + 1}}\right)  \\ \varphi^{\left(i\cdot a\right)}\left(\frac{x}{2^{n\cdot a}} + \frac{1}{2^{n\cdot a + 2}}\right) 
\end{pmatrix}~,
\end{equation*}
We also define $G_i:\reals^5\rightarrow\reals^5$ as:
\[
G_i:=F_i\circ f_{i,a}\circ\cdots\circ f_{i,1}~.
\]

We also define the input and output networks $G_{\text{in}}:\reals\rightarrow\reals^5$, $G_{\text{out}}:\reals^5\rightarrow\reals$ as:
\begin{align*}
    &G_{\text{in}}\left(x\right) = \begin{pmatrix}
    x \\ 0  \\ 0 \\ \frac{x}{2^{n\cdot a}}+  \frac{1}{2^{n\cdot a + 1}} \\ \frac{x}{2^{n\cdot a}} +  \frac{1}{2^{n\cdot a + 2}}
    \end{pmatrix} \\
    &G_{\text{out}}\left(\begin{pmatrix}
    x \\  x_{\text{cur}} \\ y\\ z_1 \\ z_2 
    \end{pmatrix}\right) = y~.
\end{align*}
Finally, we define the network $\Ncal:\reals\rightarrow\reals$ as:
\[
\Ncal := G_{\text{out}} \circ G_n\circ \cdots \circ G_1 \circ G_{\text{in}}~.
\]
We have for every $i\in[n]$ that:
\[
\bin_{(j-1)\cdot a + 1:j\cdot a}(\Ncal(x)) = \left\lfloor\bin_{(j-1)\cdot a + 1:j\cdot a}(x) \cdot 2^{-a_0}\right\rfloor~.
\]

By a similar calculation to that in \lemref{lem:single neuron}, the width of $\Ncal$ can be bounded by $7$, the depth of $\Ncal$ can be bounded by $O(n\cdot a)$, and the weights of $\Ncal$ can be bounded by $O\left(2^{n\cdot a}\right)$.

\end{proof}

\subsection{Approximation of an entire network}

We are now ready to prove the main theorem. For convenience, we restate it and also indicate the bound on the magnitude of the weights:

\begin{theorem}\label{thm:wide to narrow appen}
Let $A,B,n,L,d\in\naturals,~ \epsilon,\delta>0$ and let $\Ncal_0:[-A,A]^d\rightarrow\reals$ be a neural network with width $n$, depth $L$ and weights bounded in $[-B,B]$. Let $\Dcal$ be some distribution over $[-A,A]^d$ with density function $p_\Dcal$ such that $p_\Dcal(\bx) \leq \beta$ for every $\bx\in[-A,A]^d$ where $\beta >0$. Then, there exists a neural network $\Ncal:[-A,A]^d\rightarrow\reals$ with width $\max\{5d,10\}$, depth $O\left(n^2L^2\log(ABn\epsilon^{-1})\right)$ and weights bounded by $O\left(\frac{2^{6Ln}d^2nAB\beta(2A)^d}{\epsilon\delta}\right)$, such that $w.p > 1-\delta$ over $\bx\sim \Dcal$ we have that:
\[
|\Ncal(\bx) - \Ncal_0(\bx)| \leq \epsilon~.
\]
The total number of parameters in $\Ncal$ is $O\left(n^2L^2\log(ABn\epsilon^{-1})\right)$.
\end{theorem}

\begin{proof}[Proof of \thmref{thm:wide to narrow appen}]
We split the proof into four parts. We first describe the construction of the network which simulates the target network. Then, we bound the error of our construction. Next, we discuss how to generalize this construction to different data distributions. Finally, we calculate the size of the network we constructed.

\subsubsection*{The Construction}
We let $c_0=2L\log\left(5ABnd\epsilon^{-1}\right),~ c=2c_0 +\log(2Ad) + L\log((n+1)B)$. We will first assume that $\bx\sim U\left([-A,A]\right)^d$, and then generalize to other distributions. We use \lemref{lem:efficient encoding parameter} to define a network $F_{\text{enc}}:\reals^d\rightarrow\reals$ such that w.p $>1-\delta$ over $\tilde{\bx} \sim U\left([0,2A]^d\right)$ we get for every $i\in[d]$ that:
\begin{align}\label{eq:encoding of the data}
    &\bin_{(i-1)\cdot c + 1:i\cdot c} \left(F_{\text{enc}}(\tilde{\bx})\right) = \left\lfloor\frac{\tilde{x}_i}{2A}\cdot 2^{c_0}\right\rfloor\cdot 2A ~.
\end{align}
Denote the weight matrix of the network $\Ncal_0$ at layer $\ell$ as $W^{\ell} =\{w_{i,j}^{\ell}\}_{i,j} $, and the bias terms as $\bb^{\ell} = \{b_i^\ell\}_{i}$. We denote for each $i,j,\ell$: $\tilde{w}_{i,j}^\ell:= \left\lfloor2^{c_0}\cdot \left|w_{i,j}^\ell\right|\right\rfloor$, $\alpha_{i,j}^\ell:= \sign\left(w_{i,j}^\ell\right)$. For $\ell >1$ and for every $i$ we define: $\tilde{b}_{i}^\ell:= \left\lfloor2^{2c_0}\cdot b_{i}^\ell\right\rfloor$, for $\ell=1$ and for every $i$ we define: $\tilde{b}_i^1 := \left\lfloor2^{2c_0}\cdot b_{i}^1\right\rfloor  - A2^{c_0}\sum_{j=1}^d\alpha_{i,j}^1\tilde{w}_{i,j}^1 $. These parameters will be used in simulating each layer of the target network.

We now use \lemref{lem: approx a layer} to construct networks $F_1,\dots,F_{L-1}$ where each $F_\ell$ is used to transform the encoding of the output of the $(\ell-1)$-th layer (where the $0$-th layer is the input) to the encoding of the output of the $\ell$-th layer. We also construct $F_L$ using \lemref{lem:single neuron}, where the output of the network is without an activation, i.e. it is a linear output layer\footnote{In the proof of \lemref{lem:single neuron} this only requires to change the output of the subnetwork $G_{\text{out}}$ to not having an activation, which is possible since this is the output layer.}. The reason that $F_L$ is different from the other $F_\ell$'s, is because it needs to simulate the output layer of the target network, which does not have an activation, hence also the output layer of our construction should not have an activation. We use $c$ bits for all those constructions where $c$ is defined above, and $n_1=d, n_2=n$ for $\ell=1$, $n_1=n, n_2=n$ for $\ell=2,\dots,L-1$. The construction of each $F_\ell$ with \lemref{lem: approx a layer} and \lemref{lem:single neuron} uses the weights $\left\{\tilde{w}_{i,j}^\ell\right\}_{i,j},~ \left\{\tilde{\alpha}_{i,j}^\ell\right\}_{i,j}$ and $\left\{\tilde{b}_i^\ell\right\}_i$. For every $\ell\in[L-1]$ we also define $\tilde{F}_\ell:\reals\rightarrow\reals$ using \lemref{lem:compress x} with $a=c$ 
and $a_0=c_0$.
We define the input network $F_0:\reals^d\rightarrow\reals^d$ such that $F_0(\bx) := \bx + A\cdot \bm{1}$. Since $\bx\in[-A,A]^d$, then $F_0(\bx)\in[0,2A]$.

Finally, we define the network $\Ncal:\reals^d\rightarrow\reals$ as:
\[
\Ncal:= 2^{-2c_0} \cdot F_L\circ\tilde{F}_{L-1} \circ F_{L-1} \circ\cdots \circ \tilde{F}_1\circ F_1\circ F_{\text{enc}}\circ F_0~.
\]
We note that we do not need to also construct $\tilde{F}_L$, because the output of the last layer in the target network has dimension 
$1$. Hence, $\tilde{F}_L$ 
can be replaced with a division of the output of $F_L$ by a factor of $2^{c_0}$, 
which can be done without constructing another subnetwork. We also divide the output by another factor of $2^{c_0}$, the reason for this will be clearer in the proof of correctness section.

Note that to construct each $F_\ell$, the condition in \lemref{lem: approx a layer} should be satisfied. We prove that this condition is satisfied using induction on the number of layers. If we sample $\bx\sim U\left([-A,A]^d\right)$, then $F_0(\bx)\sim U\left([0,2A]^d\right)$. We assume throughout the induction proof that we sampled $\bx$ such that \eqref{eq:encoding of the data} holds for $\tilde{\bx} := F_0(\bx)$. Note that it holds w.p $>1-\delta$. 

For $j\in[n]$ and $\ell\in[L-1]$ denote $z_j^{\ell}:=\bin_{(j-1)\cdot c +1:j\cdot c}\left(\tilde{F}_{\ell} \circ F_{\ell}\circ \cdots \circ\tilde{F}_1 \circ F_1\circ F_{\text{enc}}\circ F_0(\bx)\right)$, and for $j\in[d]$ and $\ell=0$ denote $z_j^{\ell}:=\bin_{(j-1)\cdot c +1:j\cdot c}\left( F_{\text{enc}}\circ F_0(\bx)\right)=\bin_{(j-1)\cdot c +1:j\cdot c}\left( F_{\text{enc}}(\tilde{\bx})\right)$. To show that the condition in \lemref{lem: approx a layer} holds for any $\ell\in[L]$ we need to show that:
\[
\len\left( \left|\sum_{j=1}^n \tilde{w}_{i,j}^\ell z_j^{\ell-1} + \tilde{b}_i^\ell\right|  \right)\leq c~.
\]
For $\ell=1$ we have that:
\begin{align*}
     \left|\sum_{j=1}^n \tilde{w}_{i,j}^1 z_j^{0} + \tilde{b}_i^1\right|  &= \left|\sum_{j=1}^n \tilde{w}_{i,j}^1 \cdot \left\lfloor\frac{\tilde{x}_j}{2A}\cdot 2^{c_0}\right\rfloor\cdot 2A + \tilde{b}_i^1\right| \\
    & \leq n\cdot 2^{c_0}B\cdot  2^{c_0}\cdot 2A + |\tilde{b}_i^1| \\
    & \leq  n\cdot 2^{c_0}B\cdot 2^{c_0}\cdot 2A +  2\cdot 2^{2c_0}dAB \\
    & \leq 2\cdot2^{2c_0}AB(n+d) < 2\cdot2^{2c_0}ABnd~.
\end{align*}
Which means that:
\[
\len\left( \left|\sum_{j=1}^n \tilde{w}_{i,j}^1 \cdot \left\lfloor\frac{\tilde{x}_j}{2A}\cdot 2^{c_0}\right\rfloor\cdot 2A + \tilde{b}_i^1\right|\right) \leq 2c_0 + \log(2ABnd) \leq c~.
\]

Hence the condition of \lemref{lem: approx a layer} is satisfied for $\ell=1$. Note that by the above calculation we have that $\len\left(z_j^1\right)\leq c_0 + \log(2ABnd) $, due to the compression $\tilde{F}_1$. 

Assume that $\len\left(z_j^{\ell-1}\right)\leq c_0 + \log(2Ad) + (\ell-1)\log(B(n+1))$ for every $j$. For the $\ell$-th layer we have that:

\begin{align*}
\left|\sum_{j=1}^n \tilde{w}_{i,j}^\ell z_j^{\ell-1} + \tilde{b}_i^\ell\right| &\leq nB2^{c_0}\max_{j\in[n]}\left|z_j^{\ell-1}\right| + 2^{2c_0}B \\
&\leq nB2^{c_0}\cdot 2Ad2^{c_0}(B(n+1))^{\ell-1} + 2^{2c_0}B \\
& < 2Ad\cdot2^{2c_0}\cdot(B(n+1))^\ell~.
\end{align*}
Hence we have that:
\begin{align*}
    \len\left(\left|\sum_{j=1}^n \tilde{w}_{i,j}^\ell z_j^{\ell-1} + \tilde{b}_i^\ell\right|\right) \leq 2c_0  + \log(2Ad) + \ell\log(B(n+1)) \leq c~,
\end{align*}
which satisfies the condition of \lemref{lem: approx a layer}. Also, after applying $\tilde{F}_\ell$ we have for every $j\in[n]$ that:
\[
\len\left(z^\ell_j\right)\leq c_0 + 2\log(2Ad) + \ell\log(B(n+1))
\]
which finishes the induction proof. This shows that the construction of $\Ncal$ defined above is valid.

\subsubsection*{Proof for the Correctness of the Construction}
We now turn to bound the error. Denote by $\Ncal_0^\ell(\bx)$ the output of the target network $\Ncal_0$ after $\ell$ layers, and its $i$-th coordinate as $\Ncal_0^\ell(\bx)_i$. As before, denote by $z_j^{\ell}:=\bin_{(j-1)\cdot c +1:j\cdot c}\left(\tilde{F}_{\ell} \circ F_{\ell}\circ \cdots \circ\tilde{F}_1 \circ F_1\circ F_{\text{enc}}\circ F_0(\bx)\right)$, and for $\ell=0$ we denote: $z_j^{0}:=\bin_{(j-1)\cdot c +1:j\cdot c}\left( F_{\text{enc}}\circ F_0(\bx)\right)$. We will use throughout the proof the fact that for any $x\in\reals$ we have that $|x-\lfloor x\rfloor| \leq 1$. We will bound the error using induction on the number of layers. In particular we will show that:
\begin{align*}
     \left| \Ncal_0^{\ell}(\bx)_i - 2^{-c_0}z_i^{\ell}\right|  \leq \frac{(5nB)^\ell A\sqrt{d}}{2^{c_0}}~.
\end{align*}
As before, we assume throughout the induction proof that we sampled $\bx$ such that \eqref{eq:encoding of the data} holds for $\tilde{\bx} := F_0(\bx)$, which happens w.p $>1-\delta$. 
The intricate part is the base case since it involves the encoding of the data, and the first layer which is slightly different from the other layers. 
For
$\tilde{x}_i:=x_i + A$ 
we
have for any $i\in[d]$:
\begin{align}\label{eq:floor input}
    &\left|\tilde{x}_i - 2^{-c_0}\cdot z_i^0\right|  = \left|\tilde{x}_i - \left\lfloor\frac{\tilde{x}_i}{2A}\cdot 2^{c_0}\right\rfloor\cdot \frac{2A}{2^{c_0}}\right| \\
    & \leq \left|\tilde{x}_i - \frac{\tilde{x}_i}{2A}\cdot 2^{c_0}\cdot \frac{2A}{2^{c_0}}\right| + \left|\frac{2A}{2^{c_0}} \right| \nonumber\\
    & = \left| \frac{2A}{2^{c_0}}\right|~. \nonumber
\end{align}

In other words, the network $F_{\text{enc}}$ encodes each coordinate up to an error of $\frac{2A}{2^{c_0}}$. Let $i\in[n]$, then the error for the first layer of $\Ncal_0$ can be bounded in the following way:

\begin{align}
    \left| \Ncal_0^1(\bx)_i - 2^{-c_0}z_i^1\right| & =  \left| \Ncal_0^1(\bx)_i - 2^{-c_0}\cdot \left\lfloor2^{-c_0}\sigma\left(\sum_{j=1}^d \alpha_{i,j}^1 \tilde{w}_{i,j}^1\cdot z_j^0 + \tilde{b}_i^1\right)\right\rfloor\right| \nonumber\\
    & \leq \left| \Ncal_0^1(\bx)_i -  2^{-2c_0}\sigma\left(\sum_{j=1}^d \alpha_{i,j}^1 \tilde{w}_{i,j}^1\cdot z_j^0 + \tilde{b}_i^1\right)\right| + 2^{-c_0} \nonumber\\
    & =  \left| \sigma\left(\sum_{j=1}^d w^1_{i,j}x_j + b^1_i\right) -  \sigma\left(2^{-2c_0}\left(\sum_{j=1}^d \alpha_{i,j}^1 \tilde{w}_{i,j}^1\cdot z_j^0 + \tilde{b}_i^1\right)\right)\right| + 2^{-c_0} \nonumber\\
    & \leq \left| \sum_{j=1}^d w^1_{i,j}x_j + b^1_i -  2^{-2c_0}\left(\sum_{j=1}^d \alpha_{i,j}^1 \tilde{w}_{i,j}^1\cdot z_j^0 + \tilde{b}_i^1\right)\right| + 2^{-c_0} \nonumber
\end{align}
where the last inequality is since we use the ReLU activation which is 1-Lipschitz. We now have:
\begin{align}\label{eq:first error term}
     &\left| \sum_{j=1}^d w^1_{i,j}x_j + b^1_i -  2^{-2c_0}\left(\sum_{j=1}^d \alpha_{i,j}^1 \tilde{w}_{i,j}^1\cdot z_j^0 + \tilde{b}_i^1\right)\right| + 2^{-c_0} \nonumber\\
     = & \left| \sum_{j=1}^d w^1_{i,j}x_j + b^1_i -  2^{-c_0}\sum_{j=1}^d \alpha_{i,j}^1 \tilde{w}_{i,j}^1\cdot (2^{-c_0}z_j^0 + \tilde{x}_j - \tilde{x}_j) - 2^{-2c_0}\tilde{b}_i^1\right| + 2^{-c_0} \nonumber\\
     \leq &\left| \sum_{j=1}^d w^1_{i,j}x_j + b^1_i -  2^{-c_0}\sum_{j=1}^d \alpha_{i,j}^1 \tilde{w}_{i,j}^1\cdot \tilde{x}_j - 2^{-2c_0}\tilde{b}_i^1\right| + \nonumber\\
     & + 2^{-c_0} + 2^{-c_0}\left|\sum_{j=1}^d \alpha_{i,j}^1 \tilde{w}_{i,j}^1\cdot \left|\tilde{x}_j-2^{-c_0}z_j^0\right| \right|~. 
\end{align}
Before bounding the error terms, we will focus on the first term of \eqref{eq:first error term} and bound it further:

\begin{align}\label{eq: second error term}
    &\left| \sum_{j=1}^d w^1_{i,j}x_j + b^1_i -  \sum_{j=1}^d 2^{-c_0}\alpha_{i,j}^1 \tilde{w}_{i,j}^1\cdot \tilde{x}_j - 2^{-2c_0}\tilde{b}_i^1\right| \nonumber\\
    = & \left| \sum_{j=1}^d w^1_{i,j}x_j + b^1_i -  \sum_{j=1}^d 2^{-c_0}\alpha_{i,j}^1 \tilde{w}_{i,j}^1\cdot (x_j+ A) - 2^{-2c_0}\left\lfloor2^{2c_0}\cdot b_{i}^1\right\rfloor  + A2^{-c_0}\sum_{j=1}^d\alpha_{i,j}^1\tilde{w}_{i,j}^1\right| \nonumber \\
    = &  \left| \sum_{j=1}^d w^1_{i,j}x_j + b^1_i -  \sum_{j=1}^d 2^{-c_0}\alpha_{i,j}^1 \left\lfloor2^{c_0}\cdot \left|w_{i,j}^1\right|\right\rfloor x_j - 2^{-2c_0}\left\lfloor2^{2c_0}\cdot b_{i}^1\right\rfloor \right| \nonumber\\
    =&  \left| \sum_{j=1}^d w^1_{i,j}x_j + b^1_i -  \sum_{j=1}^d 2^{-c_0}\alpha_{i,j}^1 \cdot 2^{c_0} \left|w_{i,j}^1\right| x_j \right.\nonumber\\ 
     -& \left.\sum_{j=1}^d 2^{-c_0}\alpha_{i,j}^1 \cdot \left(\left\lfloor2^{c_0}\cdot \left|w_{i,j}^1\right|\right\rfloor -  2^{c_0} \left|w_{i,j}^1\right|\right)x_j  - 2^{-2c_0}\left(\left\lfloor  2^{2c_0}\cdot b_{i}^1\right\rfloor + 2^{2c_0}\cdot b_{i}^1 - 2^{2c_0}\cdot b_{i}^1\right) \right| \nonumber\\
    \leq& \left| \sum_{j=1}^d w^1_{i,j}x_j + b^1_i -  \sum_{j=1}^d w^1_{i,j}x_j - b^1_i\right|  +  2^{-2c_0}  + \left|\sum_{j=1}^d 2^{-c_0}\alpha_{i,j}^1 x_j \right| \nonumber \\
    = & 2^{-2c_0}  + \left|\sum_{j=1}^d 2^{-c_0}\alpha_{i,j}^1 x_j \right|~.
\end{align}
Combining \eqref{eq:first error term} and \eqref{eq: second error term}, we bound the error for the first layer in the following manner:
\begin{align*}
    \left| \Ncal_0^1(\bx)_i - 2^{-c_0}z_i^1\right| &\leq  2^{-c_0} +  2^{-2c_0} + 2^{-c_0}\left|\sum_{j=1}^d \alpha_{i,j}^1 \tilde{w}_{i,j}^1\cdot \left|\tilde{x}_j-2^{-c_0}z_j^0\right| \right|   + \left|\sum_{j=1}^d 2^{-c_0}\alpha_{i,j}^1 x_j \right| \\
    &\leq  2^{-c_0} +  2^{-2c_0} +2^{-c_0}dB\cdot \frac{2A}{2^{c_0}} + dA2^{-c_0} \leq \frac{5dAB}{2^{c_0}} \leq \frac{5nAB\sqrt{d}}{2^{c_0}}~,
\end{align*}
where in the second inequality we used \eqref{eq:floor input}, and in the last inequality we used that $d\leq n$. This finishes the base case for the induction. 
We now bound the error for the output of the $\ell$-th layer in a similar way to the error bound of the first layer.  Assume that for the $\ell-1$ layer we have:
\begin{align*}
     \left| \Ncal_0^{\ell-1}(\bx)_i - 2^{-c_0}z_i^{\ell-1}\right|  \leq \frac{(5nB)^{\ell-1} A\sqrt{d}}{2^{c_0}}~.
\end{align*}
For conciseness we do not repeat all the inequalities in details and just state the final bound. The derivations here are similar to the ones done for the base case.
\begin{align}\label{eq:error bound inductive case}
    \left| \Ncal_0^\ell(\bx)_i - 2^{-c_0}z_j^{\ell}\right|  \leq 2^{-c_0} +  2^{-2c_0}  + 2^{-c_0}\left|\sum_{j=1}^n \alpha_{i,j}^\ell \tilde{w}_{i,j}^\ell\cdot \left|\Ncal_0^{\ell-1}(\bx)_j-2^{-c_0}z_j^{\ell-1}\right| \right| + \left|\sum_{j=1}^n 2^{-c_0}\alpha_{i,j}^1 \Ncal_0^{\ell-1}(\bx)_j \right|
\end{align}
Note that for $\ell=L$, the layer is slightly different because it does not have an activation, and also because instead of applying $\tilde{F}_L$ we divide the output by a factor of $2^{c_0}$. Note that since the output of the $L$-th layer is of dimension 1, then applying $\tilde{F}_L$ is equivalent to dividing by a factor of $2^{c_0}$ and applying the integral part function. By carefully following the calculations for the approximation error (with some minor modifications), it can be seen that removing the activation and the integral part function also implies a similar guarantee for the last layer. 



To bound the output of the $(\ell-1)$-th layer $\Ncal_0^{\ell-1}(\bx)_j$ we use a rough estimate of the Lipschitz parameter of the network in the following way:
\begin{align*}
    \left\|\Ncal_0^{\ell}(\bx)\right\|  &= \left\|\sigma\left(W^{\ell}\Ncal_0^{\ell-1}(\bx) + \bb^{\ell}\right)\right\| \leq  \left\|W^{\ell}\Ncal_0^{\ell-1}(\bx) + \bb^{\ell}\right\| \\
    & \leq \left\| W^\ell\right\|\cdot  \left\| \Ncal_0^{\ell-1}(\bx)\right\|  + \left\|\bb^{\ell}\right\| \leq \left\| W^\ell\right\|_F\cdot  \left\| \Ncal_0^{\ell-1}(\bx)\right\|  + \left\|\bb^{\ell}\right\| \\
    & \leq nB\cdot \left\| \Ncal_0^{\ell-1}(\bx)\right\| + \sqrt{n}\cdot B \leq 2nB\cdot \left\| \Ncal_0^{\ell-1}(\bx)\right\|~.
\end{align*}
Using the above inductively, that $d\leq n$ and that $\left\| \Ncal_0^{0}(\bx)\right\| = \norm{\bx} \leq A\cdot \sqrt{d}$, to get that:
\begin{align*}
    \left\|\Ncal_0^{\ell}(\bx)\right\| \leq (2nB)^\ell\cdot A\sqrt{d}~.
\end{align*}
Since $\left|\Ncal_0^{\ell-1}(\bx)_j\right| \leq \left\|\Ncal_0^{\ell-1}(\bx)\right\|$ we can plug the above bound, and the inductive assumption into \eqref{eq:error bound inductive case} and using that $d\leq n$ we have:
\begin{align*}
    \left| \Ncal_0^\ell(\bx)_i - 2^{-c_0}z_j^{\ell}\right|  \leq 2\cdot 2^{-c_0} + nBA\sqrt{d}2^{-c_0}\cdot(5nB)^{\ell-1} + nA\sqrt{d}2^{-c_0} \cdot (2nB)^{\ell-1}\leq \frac{(5nB)^{\ell}A\sqrt{d}}{2^{c_0}}~.
\end{align*}
This finishes the induction proof. In particular, after $L$ layers by the definition of $c_0$ we get that:
\[
\left|\Ncal_0(\bx) - \Ncal(\bx)\right|  \leq \frac{(5nB)^LA \sqrt{d}}{2^{c_0}} \leq \epsilon~,
\]
which gives us the required error for our construction.

\subsubsection*{Different Data Distributions}
Suppose we are given some distribution $\Dcal$ over $[-A,A]^d$ with density function such that $P_\Dcal(\bx)\leq \beta$ for every $\bx\in [-A,A]^d$. Let $C\subseteq [-A,A]^d$ be the set for which \eqref{eq:encoding of the data} does not hold for $\tilde{x}_i = x_i + A$ for some $i\in[d]$. 
By our construction, if we sample $\bx\sim U([-A,A]^d)$, then \eqref{eq:encoding of the data} holds for $\tilde{\bx}$ w.p $ > 1-\delta$. Hence, we have that $\text{Vol}(C)\leq \delta\cdot (2A)^d$. By the assumption on $\Dcal$ we get that 
\[
\text{Pr}_\Dcal(\bx\in C) \leq \beta\cdot \text{Vol}(C)\leq \beta\delta(2A)^d~.
\]
Hence, 
given $\delta>0$ we can construct a network $\Ncal$ that approximates $\Ncal_0$ w.p $> 1-\delta'$ over $U([-A,A]^d)$ for $\delta' := \frac{\delta}{\beta(2A)^d}$, and by the above equation the network $\Ncal$ approximates $\Ncal_0$ w.p $> 1-\delta$ over $\Dcal$.
Note that replacing $\delta$ by $\delta'$ in our construction only affects the size of the weights, and not the width, depth or the number of parameters in the network (see a detailed calculation below).

\subsubsection*{The Size of the Constructed Network}
We will now calculate the size of the network $\Ncal$. The width of the network is the maximal width of its subnetworks. The width of $F_{\text{enc}}$ can be bounded by $5d$. The width of $F_0$ can be bounded by $d$, since it is only a translation of each input coordinate by $A$. The width of each $F_i$ can be bounded by $10$ and of each $\tilde{F_i}$ can be bounded by 7. In total, the width of $\Ncal$ can be bounded by $\max\{5d, 10\}$. 

The depth of $\Ncal$ can be bounded by the sum of the depths of its subnetworks. The depth of $F_{\text{enc}}$ can be bounded by $O(c_0) = O\left(L\log(ABn\epsilon^{-1})\right)$. The depth of $F_0$ is $1$. The depth of $F_1$ can be bounded by $O(dnc) = O\left(dnL\log(ABn\epsilon^{-1})\right)$ and the depth of each $F_\ell$ for $\ell\in\{2,\dots,L\}$ can be bounded by $O(n^2c) = O\left(n^2L\log(ABn\epsilon^{-1})\right)$. The depth of each $\tilde{F}_\ell$ for $\ell\in[L-1]$ can be bounded by $O(c) = O\left(L\log(ABn\epsilon^{-1})\right)$. In total, the depth of $\Ncal$ can be bounded by $O\left(n^2L^2\log(ABn\epsilon^{-1})\right)$ where we used that $d\leq n$.

The weights of $\Ncal$ can be bounded by the largest bound on the weights of its subnetworks. The weights of $F_{\text{enc}}$ can be bounded by 
\begin{align*}
O\left(\frac{2^{dc}d\cdot \beta(2A)^d}{\delta}\right) &= O\left(\frac{2^{4dL\log(5ABn\sqrt{d}\epsilon^{-1}) + d\log(2Ad) + dL\log((n+1)B)}d\cdot \beta(2A)^d}{\delta}\right) \\
&= O\left(\frac{2^{6dL\log(5ABnd\epsilon^{-1})}d\beta (2A)^d}{\delta}\right) = O\left(\frac{2^{6dL}d^2ABn\beta (2A)^d}{\epsilon\delta}\right)~.
\end{align*}

The weights of $F_\ell$ for $\ell\geq 2$ can be bounded by $O(2^{nc}) = O\left(2^{6nL\log(5ABnd\epsilon^{-1})}\right) = O\left(\frac{2^{6nL}ABnd}{\epsilon}\right)$. The weights of $F_1$ can be bounded by $O(2^{dc})$ which are smaller than the weights of $F_\ell$ for $\ell\geq 2$. The weights of each $\tilde{F}_\ell$ for $\ell\in[L-1]$  can be bounded by the same bound. In total, using the assumption that $d\leq n$ we can bound the weights of $\Ncal$ by $O\left(\frac{2^{6Ln}d^2nAB\beta(2A)^d}{\epsilon\delta}\right)$.

The number of parameters in the network $\Ncal$ can be bounded by the sum of the number of parameters in each of its subnetworks. For each subnetwork we bound the number of parameters by its depth times the square of its width. The number of parameters in $F_{\text{enc}}$ can be bounded by $O\left(d^2L\log(ABn\epsilon^{-1})\right)$. The number of parameters of each $F_\ell$ for $\ell\in [L]$ can be bounded by $O\left(n^2L\log(ABn\epsilon^{-1})\right)$. The number of parameters for each $\tilde{F}_\ell$ for $\ell\in[L]$ can be bounded by $O\left(L\log(ABn\epsilon^{-1})\right)$. The number of parameters in $F_0$ is $O(d^2)$. Hence, the total number of parameters in $\Ncal$ can be bounded by $O\left(n^2L^2\log(ABn\epsilon^{-1})\right)$, where we used that $d\leq n$.
\end{proof}

\section{Proofs from \secref{sec:opt width}}\label{appen:proofs from opt width}
The following lemma improves on \lemref{lem:efficient encoding parameter} in terms of width, but the required depth is larger:

\begin{lemma}\label{lem:efficient width}
Let $\delta>0,~c,c_0,A\in\naturals$, where $c\geq c_0 +  \log(A) +1$. There exists a neural network $\Ncal:\reals^d\rightarrow\reals$ with width $d+2$ depth at most $O(c_0^2d)$ and weights bounded by $ \left(\frac{2^cAd}{\delta}\right)$, such that if we sample $\bx\sim U([0,A]^d)$, then w.p $>1-\delta$ for every $i\in[d]$ we have that:
\begin{align*}
    &\bin_{(i-1)\cdot c + 1:i\cdot c} \left(\Ncal(\bx)\right) = \left\lfloor\frac{x_i}{A}\cdot 2^{c_0}\right\rfloor\cdot A~. 
\end{align*}
\end{lemma}

\begin{proof}
In order to use less neurons in each layer of the network, we will use a slightly different bit extraction technique, which was also used in \cite{safran2017depth}. This construction is presented in \lemref{lem:ineqfficient Telgarsky}. We first construct the first layer which divides all the inputs by $A$, that is:
\[
\begin{pmatrix}
x_1 \\ \vdots \\ x_d
\end{pmatrix} \mapsto \begin{pmatrix}
\frac{x_1}{A} \\ \vdots \\ \frac{x_d}{A}
\end{pmatrix}~.
\]
This way, we can assume that the input is in $[0,1]^d$. We use \lemref{lem:ineqfficient Telgarsky} to define the following subnetwork $F:[0,1]^2\rightarrow\reals$ with width $4$ which maps $F\left(\begin{pmatrix}
x \\ y
\end{pmatrix}\right) = \left\lfloor x\cdot 2^{c_0} \right\rfloor +y$ w.p $>1-\frac{\delta}{d}$.

We will define the subnetworks $F_1,\dots,F_d$ which will extract the relevant bits from each input coordinate, and add it to the output coordinate. The construction of $F_1$ will be slightly different from the construction of the other $F_i$'s since we constrain ourselves to having a width of $d+2$, hence we will "hide" the extracted bits from the first coordinate inside the second coordinate. We will then extract those bits and add them to a designated output neuron. 

Concretely, we define $F_1':\reals^d\rightarrow\reals^{d-1}$ such that:
\[
F'_1\left(\begin{pmatrix}
x_1 \\ \vdots \\ x_d
\end{pmatrix}\right) = \begin{pmatrix}
F\left(\begin{pmatrix}
x_1 \\ x_2
\end{pmatrix}\right) \\ x_3 \\ \vdots \\ x_d
\end{pmatrix}
\]

In the first coordinate of the output of $F_1'$ we have both the value of $x_2$ in the fractional part, and the bits we extracted from $x_1$ in the integral part. We use a slightly different version of the network $F$, which we call $F'$ such that: $F'\left(
x\right) = \begin{pmatrix}
 \left\lfloor x \right\rfloor  \\ x
\end{pmatrix}$~. Note that constructing $F'$ is similar to the construction of $F$, where we take the integral value of the input. For conciseness we do not repeat this construction. It is also a width $4$, depth $O(c_0)$ network. Next, we define $F_1'':\reals^{d-1}\rightarrow \reals^d$ such that: 
\[
F''_1\left(\begin{pmatrix}
x_2 \\ \vdots \\ x_d
\end{pmatrix}\right) = \begin{pmatrix}
F'\left(x_2
\right) \\ x_3 \\ \vdots \\ x_d
\end{pmatrix}~.
\]
We also define the network $G:\reals^d\rightarrow\reals^d$ as:
\[
G\left(\begin{pmatrix}
y \\ x_2 \\ \vdots \\ x_d
\end{pmatrix}\right)  = \begin{pmatrix}
A\cdot y \\  x_2 - y \\ x_3 \\ \vdots \\ x_d
\end{pmatrix}~.
\]
We define the network $F_1:= G\circ F_1''\circ F_1'$. By the construction we get that:
\begin{align*}
    F_1\left(\begin{pmatrix}
x_1/A \\ \vdots \\ x_d/A
\end{pmatrix}\right) & = G\circ F_1''\left(\begin{pmatrix}
\left\lfloor x_1/A\cdot 2^{c_0}\right\rfloor + x_2/A\\ x_3/A \\ \vdots \\ x_d/A 
\end{pmatrix}\right) \\
& =G\left(\begin{pmatrix}
\left\lfloor x_1/A\cdot 2^{c_0}\right\rfloor \\  \left\lfloor x_1/A\cdot 2^{c_0}\right\rfloor + x_2/A \\ x_3/A \\ \vdots \\ x_d/A
\end{pmatrix}\right) =  \begin{pmatrix}
\left\lfloor\frac{x_1}{A}\cdot 2^{c_0}\right\rfloor\cdot A  \\ x_2/A \\ \vdots \\ x_d/A
\end{pmatrix}~.
\end{align*}
where the second equality is since $x_2/A < 1$ w.p 1, hence $\left\lfloor \left\lfloor x_1/A\cdot 2^{c_0}\right\rfloor + x_2/A \right\rfloor = \left\lfloor x_1/A\cdot 2^{c_0}\right\rfloor$.

Now we define $F_2,\dots,F_d$ such that $F_i:\reals^{d-i+2}\rightarrow\reals^{d-i+1}$ where:
\[
F_i\left(\begin{pmatrix}
y \\x_i \\ \vdots \\ x_d
\end{pmatrix}\right) = \begin{pmatrix}
A\cdot F\left(\begin{pmatrix}
x_i \\ 0
\end{pmatrix}\right)  + y\cdot 2^c\\ x_{i+1} \\ \vdots \\ x_d
\end{pmatrix}
\]
Finally, we construct $\Ncal:\reals^d\rightarrow\reals$ such that:
\[
\Ncal := F_d\circ\cdots\circ F_1~.
\]
By the construction of $\Ncal$, and using union bound over $i\in[d]$ we have w.p $>1-\delta$  that for every $i$:  $\bin_{(i-1)\cdot c + 1:i\cdot c} \left(\Ncal(\bx)\right) = \left\lfloor\frac{x_i}{A}\cdot 2^{c_0}\right\rfloor\cdot A $ as required.
As was argued before, the construction of $F_1$ has width $d+2$. In addition, the construction of each $F_i$ for $i\geq 2$ has width $d-i+4\leq d+2$. Hence, the network $\Ncal$ has width at most $d+2$ as required. The depth of each $F_i$ is $O(c^2)$, hence the depth of the network $\Ncal$ is $O(c^2d)$. The maximal weight of $\Ncal$ can be bounded by $\left(\frac{2^cAd}{\delta}\right)$.
\end{proof}

The following lemma shows a bit extraction technique, which is more efficient in terms of width from \lemref{lem:enc data single coord} but less efficient in terms of depth.

\begin{lemma}\label{lem:ineqfficient Telgarsky}
Let $\delta > 0$ and $c\in\naturals$. There exists a neural network $\Ncal:[0,1]^2\rightarrow\reals$ with width 4, depth bounded by $O(c^2)$ and weights bounded by $O\left(\frac{2^c}{\delta}\right)$, such that if we sample $x\sim U\left([0,1]\right)$, then w.p $> 1-\delta$ for any $y\in[0,1]$ we have that $\Ncal\left(\begin{pmatrix}
x \\ y
\end{pmatrix}\right)= \left\lfloor x\cdot 2^{c} \right\rfloor +y$.
\end{lemma}

\begin{proof}
We define $\varphi(z) = \sigma(\sigma(2z)-\sigma(4z-2))$, this is Telgarsky's triangle function \cite{telgarsky2016benefits}. We also define the following function:
\[
 h_\delta(x) = \frac{1}{\delta}\sigma\left(x - \frac{1}{2} + \frac{\delta}{2}\right) - \frac{1}{\delta}\sigma\left(x - \frac{1}{2} - \frac{\delta}{2}\right)~.
\]
This function approximate the indicator $\mathbbm{1}(x\geq 1/2)$. That is, for every $x\notin [1/2-\delta,1/2+\delta]$ we have that $h_\delta(x) =\mathbbm{1}(x\geq 1/2)$. For every $i$ we define the function 
\[
\psi_i(x) = h_\delta\left(\varphi^{(i)}\left(x - \frac{1}{2^{i+1}}\right)\right)~.
\]
Note that if we sample $x\sim[0,1]$, then w.p $1-\delta$ we have that $\varphi^{(i)}\left(x - \frac{1}{2^{i+1}}\right)\notin [1/2-\delta,1/2+\delta]$ for every $i\in[c]$. Hence, intuitively this function extracts the $i$-th bit of $x$ w.h.p. We now define the network $F_i:[0,1]^2\rightarrow\reals^2$ such that:
\[
F_i\left(\begin{pmatrix}
x \\ y
\end{pmatrix}\right) = \begin{pmatrix}
x \\ y + 2^{c-i+1}\cdot \psi_i(x)
\end{pmatrix}
\]
This network can be constructed using width $4$, that is simulating $\psi_i$ using width 2, and keeping throughout the calculation the inputs $x,y$. Finally we construct the network $\Ncal:[0,1]^2\rightarrow[0,1]$ as:
\[
\Ncal:= P\circ F_c\circ\cdots\circ F_1
\]
where $P$ is the projection on the second coordinate. By our construction, we get that $\Ncal\left(\begin{pmatrix}
x \\ y
\end{pmatrix}\right)= \left\lfloor x\cdot 2^{c} \right\rfloor +y$ as required. The width of $\Ncal$ is the maximal width of its subnetwork which is $4$. The depth of each $F_i$ is bounded by $O(c)$, hence the depth of $\Ncal$ can be bounded by $O(c^2)$. The weights of $\Ncal$ can be bounded by $O\left(\frac{2^c}{\delta}\right)$.
\end{proof}

We are now ready to prove the main theorem in this section:
\begin{proof}[Proof of \thmref{thm:optimal width}]
We construct $\Ncal$ in the same way as the proof of \thmref{thm:wide to narrow}, where the only difference is that to construct $F_{\text{enc}}$ we use \lemref{lem:efficient width}. The correctness of the construction follows from the same arguments. For conciseness we do not repeat the entire proof, and only calculate the required width, depth and number of parameters in the network.

The width of the network $F_{\text{enc}}$ is bounded by $d+2$. The width of any other component in the network is bounded by $10$. Hence, the width of $\Ncal$ can be bounded by $\max\{d+2,10\}$.

The depth of $F_{\text{enc}}$ can be bounded by $O(c_0^2d):= O\left(L^2\log(ABn\epsilon^{-1})^2 d\right)$. Using the same bounds on the other parts of the network from the proof of  \thmref{thm:wide to narrow} we get that the depth of $\Ncal$ can be bounded by $O\left(n^2L^2\log(ABN\epsilon^{-1})\right)$.

The bound on the the weights of $\Ncal$ remains the same as in \thmref{thm:wide to narrow}.

The number of parameters in $F_{\text{enc}}$ can be bounded by $O\left(d^3L^2\log(ABn\epsilon^{-1})^2\right)$. The number of parameters for the other parts of the network remains the same as in the proof of \thmref{thm:wide to narrow}. In total, the number of parameters in $\Ncal$ can be bounded by $O\left(dn^2L^2\log(ABn\epsilon^{-1})^2\right)$. If $n \geq d^{1.5}$, then $d^3 = O(n^2)$, this means that in this case, the total number of parameters is $O\left(n^2L^2\log(ABn\epsilon^{-1})^2\right)$.
\end{proof}

\section{Proofs from \secref{sec:exact}}\label{appen:proofs from sec exact}

\subsection{The 2-Layer Case}

\begin{lemma}\label{lem:exact 2-layers}
Let $f^*:\mathbb{R}^d\rightarrow \mathbb{R}$ be a $2$-layer neural network with width $n$. Then there exists a (n+2)-layer neural network $f:\mathbb{R}^d\rightarrow \mathbb{R}$ with width $2d+2$, such that for every $\bx\in\mathbb{R}^d$ we have $f^*(\bx) = f(\bx)$.
\end{lemma}

\begin{proof}
We can write:
\begin{equation}
    f^*(\bx) = \sum_{i=1}^n u_i^*\sigma(\inner{\bw_i^*, \bx} + b_i^*) +b^*~.
\end{equation}
We denote $W^{(i)}, \bb^{(i)}$ to be the weights of the $i$-th layer of the network $f$ we construct, and $\bw^{(i)}_j$ the $j$th row of $W^{(i)}$. For convenience, we denote the first layer of $f$ as $W^{(0)}$ and define it as:
\[
W^{(0)} =  \begin{pmatrix} I_d \\-I_d \\ \bm{0}^\top_d \\ \bm{0}^\top_d
\end{pmatrix} \in\reals^{(2d+2)\times d}
\]
and $\bb^{(0)} = \bm{0}_{2d+2}$. Then, we have: $x^{(0)}:= \sigma(W^{(0)}\bx + \bb^{(0)}) =\begin{pmatrix}\sigma(\bx) \\ \sigma(-\bx) \\ 0 \\ 0 \end{pmatrix} $. Now for every $i=1,\dots,n$ we define $W^{(i)}$ and $\bb^{(i)}$ in the following way: If $u_i^* \geq 0$ then we define
\[
W^{(i)} = \begin{pmatrix} &I_d ~&\bm{0}_{d\times d} ~ &\bm{0}_d ~&\bm{0}_d \\
&\bm{0}_{d\times d} ~ &I_d ~ &\bm{0}_d ~ &\bm{0}_d \\
&u_i^{*^\top}\cdot \bw_i^{*} ~ &-u_i^{*^\top}\cdot \bw_i^{*} ~ &1 ~&0 \\
&\bm{0}_d^\top ~ &\bm{0}_d^\top ~ &0 ~ &1
\end{pmatrix} \in \mathbb{R}^{(2d+2)\times (2d+2)}~,~ \bb^{(i)} = \begin{pmatrix}\bm{0}_d\\ \bm{0}_d\\b_i^* \\ 0\end{pmatrix}\in\reals^{2d+2}
\]
otherwise, we define
\[
W^{(i)} =\begin{pmatrix} &I_d ~&\bm{0}_{d\times d} ~ &\bm{0}_d ~&\bm{0}_d \\
&\bm{0}_{d\times d} ~ &I_d ~ &\bm{0}_d ~ &\bm{0}_d \\
&\bm{0}_d^\top ~ &\bm{0}_d^\top ~ &1 ~ &0\\
&|u_i|^{*^\top}\cdot \bw_i^{*} ~ &-|u_i|^{*^\top}\cdot \bw_i^* ~ &0 ~&1 
\end{pmatrix} \in \mathbb{R}^{(2d+2)\times (2d+2)}~,~ \bb^{(i)} = \begin{pmatrix}\bm{0}_d\\ \bm{0}_d \\ 0 \\b_i^* \end{pmatrix}\in\reals^{2d+2}.
\]
Denote by $I_+:=\{i\in[n]:u_i^* \geq 0\}$ and similarly $I_-:=\{i\in[n]:u_i^* < 0\}$. Denote by $\bx^{(i)}$ the output of the network $f$ after $i$ layers (including the zeroth layer). We use the fact that for the ReLU activation we have that $\sigma(z)-\sigma(-z) = z$ for every $z\in\reals$, hence also $\inner{\bw,\sigma(\bx)} - \inner{\bw,\sigma(-\bx)} = \inner{\bw,\bx}$ for every $\bw,\bx\in\reals^d$. Then we have that:
\[
\bx^{(n)} = \begin{pmatrix}\sigma(\bx) \\ \sigma(-\bx) \\ \sum_{i\in I_+}u_i^*\sigma(\inner{\bw_i^*,\bx} + b_i) \\ \sum_{i\in I_-}|u_i|^*\sigma(\inner{\bw_i^*,\bx}+ b_i)\end{pmatrix}\in\reals^{2d+2}~.
\]
We define the last layer of the network as:
\[
W^{(n+1)}= \begin{pmatrix} \bm{0}_d\\\bm{0}_d\\1\\-1\end{pmatrix}\in\reals^{(2d+2)},~\bb^{(n+1)}=b^*\in\reals~.
\]
In total, we get that:
\begin{align*}
    \bx^{(n+1)} &= b^* + \sum_{i\in I_+}u_i^*\sigma(\inner{\bw_i^*,\bx} + b_i^*) - \sum_{j\in I_-}|u_j|^*\sigma(\inner{\bw_j^*,\bx} + b_j^*) \\
    &= \sum_{i=1}^n u_i^*\sigma(\inner{\bw_i^*, \bx} + b_i^*) +b^*=  f^*(\bx) ~.
\end{align*}

\end{proof}

\subsection{The General Case}

\begin{proof}[Proof of \thmref{thm: exact}]
We will show by induction on the depth $L$ that for every neural network $\Ncal^*:\reals^d\rightarrow\reals$ of width at most $n$ and depth $L$ there is another neural network $N:\reals^{2(d+L-1)}\rightarrow\reals^{2(d+L-1)}$ of width $2(d+L-1)$ and depth $(2n)^{L-1}$, such that for $N_0^L:\reals^d\rightarrow\reals^{2(d+L-1)}$ defined by the weights:
\begin{equation}\label{eq:N_0 weights}
W_0 = \begin{pmatrix}I_d \\ -I_d \\ \bm{0}_{(2L-2)\times d}\end{pmatrix}\in\reals^{2(d+L-1)\times d}, ~ \bb_0=\bm{0}_{2(d+L-1)}~\in \reals^{2(d+L-1)}.
\end{equation}
we have that:
\begin{enumerate}
    \item $ (N\circ N_0^L(\bx))_{2d+2L-3} - (N\circ N_0^L(\bx))_{2d+2L-2} + b^*  = \Ncal^*(\bx)$ for every $\bx\in\reals^d$, where $b^*$ is the bias in the output layer of $\Ncal$.
    \item The $i$-th  coordinate of both the input and output of $N$ is equal to $\sigma(x_i)$ for $i\in\{1,\dots,d\}$ and $\sigma(-x_i)$ for $i\in\{d+1,\dots,2d\}$~, 
\end{enumerate}



The case of $L=2$ is proved in \lemref{lem:exact 2-layers}, by taking all but the first and last layers of the construction there. 
Suppose this is true for every $\ell < L$, and that we are given a network $\Ncal^*:\reals^d\rightarrow\reals$ of depth $L$, with weights in the last layer $\bw^L\in\reals^n,~b^L\in\reals$. For every $\bx\in\reals^d$ let $z(\bx)\in\reals^n$ be the output of the $(L-1)$-layer of $\Ncal^*$ on $\bx$. We can write $\Ncal^*(\bx)$ as a linear function over its last layer:
\begin{equation}\label{eq:second to last layer}
    \Ncal^*(\bx) = \inner{\bw^L,z(\bx)} + b^L = b^L + \sum_{i=1}^n w^L_i\cdot z(\bx)_i.
\end{equation}
Note that each coordinate of $z(\bx)$ is the output of an $(L-1)$-layer network composed with the ReLU activation. By the induction hypothesis, for every coordinate $i\in[n]$, there is a neural network $N_i$ of width $2d+2L-4$ and depth at most $(2n)^{L-2}$, such that for every $\bx\in\reals^d$ we have:

\begin{equation}\label{eq:def of z i(x)}
z(\bx)_i = \sigma\left((N_i\circ N_0^{L-1}(\bx))_{2d+2L-5} - (N_i\circ N_0^{L-1}(\bx))_{2d+2L-4} + b_i^{L-1}  \right)
\end{equation}

We construct a neural network $N$ in the following way: For each $i=1,\dots,n$ we add to the network $N_i$ two extra coordinates, and for each matrix in the network $N_i$ we concatenate it with the block matrix $I_2$. In other words, we extend each $N_i$ to have two more inputs and outputs coordinates, and just apply the identity on those two coordinates. This way, the output of each network $N_i$ on the last two coordinates are just the ReLU of the input. Now, for each $N_i$ we construct a depth-2, width-$2(d+L-1)$ network $M_i$ which maps the following input to output, where we assume in the following equations that the $y_j$'s are non-negative. If $w_i^L \geq0$ then:
\[
\begin{pmatrix}
\sigma(x_1) \\ \vdots \\ \sigma(x_d) \\ \sigma(-x_1) \\ \vdots \\ \sigma(-x_d) \\ y_1 \\ \vdots \\y_{2L-5} \\ y_{2L-4} \\ y_{2L-3} \\ y_{2L-2}
\end{pmatrix} \mapsto 
\begin{pmatrix}
\sigma(x_1) \\ \vdots \\ \sigma(x_d) \\ \sigma(-x_1) \\ \vdots \\ \sigma(-x_d) \\ 0 \\ \vdots \\0 \\ 0 \\ y_{2L-3} + w_i^L\sigma\left(y_{2L-5} - y_{2L-4} + b_i^{L-1}\right) \\ y_{2L-2}
\end{pmatrix}
\]
and if $w_i^L <0$ then:
\[
\begin{pmatrix}
\sigma(x_1) \\ \vdots \\ \sigma(x_d) \\ \sigma(-x_1) \\ \vdots \\ \sigma(-x_d) \\ y_1 \\ \vdots \\y_{2L-5} \\ y_{2L-4} \\ y_{2L-3} \\ y_{2L-2}
\end{pmatrix} \mapsto 
\begin{pmatrix}
\sigma(x_1) \\ \vdots \\ \sigma(x_d) \\ \sigma(-x_1) \\ \vdots \\ \sigma(-x_d) \\ 0 \\ \vdots \\0 \\ 0 \\ y_{2L-3} \\ y_{2L-2}  + |w_i^L|\sigma\left(y_{2L-5} - y_{2L-4} + b_i^{L-1}\right)
\end{pmatrix}~.
\]
We can define $M_i$ this way, since the output of the ReLU function is positive, and simulating the identity on positive inputs using ReLU requires a single neuron ($\sigma(z) = z$ for all $z\geq 0$).

Finally we define the network as:
\[
N(\bx) =  M_{n}\circ N_n\cdots M_1\circ N_1(\bx)~.
\]
First, note that the construction is valid, since for every $M_i$, its output on the first $2d$ coordinates satisfies the induction hypothesis, and also the output of the next 2L-4 coordinates is zero. Second, by the construction and \eqref{eq:second to last layer} and \eqref{eq:def of z i(x)} it is easy to see that for every $\bx$ we have that:
\[
\Ncal^*(\bx) = (N\circ N_0^{L}(\bx))_{2d+2L-3} - (N\circ N_0^{L}(\bx))_{2d+2L-2} + b^{L}  
\]
The depth of $N$, by the induction hypothesis, can be bounded by $2n\cdot (2n)^{L-2} = (2n)^{L-1}$. The width of $N$ by the construction is bounded by $2d+2L-2$. This finishes the induction proof.

We now define $N_{\text{out}}:\reals^{2(d+L-1)}\rightarrow\reals$ as the affine operator which maps 
\[
\begin{pmatrix}
y_1 \\ \vdots \\ y_{2(d+L-1)}
\end{pmatrix}\mapsto y_{2d+2L-3} - y_{2d+2L-2} + b^L~.
\]
We finally define the network $\Ncal:\reals^d\rightarrow\reals$ as 
\[
\Ncal:= N_{\text{out}}\circ N \circ N_0^L
\]
where $N$ is given by the induction step for $L$. We have shown that for any $\bx\in\reals^d$ we have that $\Ncal(\bx) = \Ncal^*(\bx)$.

The depth of the entire network is equal to the depth of the network $N$ plus the input and output subnetworks, which by the induction hypothesis is at most $(2n)^{L-1} + 2$, while the width of $\Ncal$ is $2d+2L-2$.
\end{proof}

\end{document}